%% file: neurips_2026.tex
\documentclass{article}

\usepackage[preprint]{neurips_2026}

\usepackage[utf8]{inputenc} 
\usepackage[T1]{fontenc}    
\usepackage{hyperref}       
\usepackage{url}            
\usepackage{booktabs}       
\usepackage{amsfonts}       
\usepackage{nicefrac}       
\usepackage{microtype}      
\usepackage{xcolor}         
\usepackage{graphicx}
\usepackage{stmaryrd}
\usepackage{amsmath}
\usepackage{amssymb}
\usepackage{mathtools}
\usepackage{amsthm}
\usepackage{svg}
\usepackage{csquotes}
\newtheorem{proposition}{Proposition}

\usepackage{amsmath}

\title{Prospective Compression in Human Abstraction Learning}

\usepackage{authblk}

\author[1]{Leonardo Hernandez Cano}
\author[2$*$]{Ivan Zareski}
\author[2$*$]{Luisa El Amouri}
\author[3$*$]{Pinzhe Zhao}
\author[2]{Max~Mascini}
\author[1]{Emanuele Sansone}
\author[4]{Yewen Pu}
\author[3$\dagger$]{Bonan Zhao}
\author[2$\dagger$]{Marta Kryven}

\affil[1]{Massachusetts Institute of Technology}
\affil[2]{Dalhousie University}
\affil[3]{University of Edinburgh}
\affil[4]{Nanyang Technological University}

\begin{document}

\maketitle

\footnotetext[1]{Equal contribution.}
\footnotetext[2]{Joint senior authors.}


\begin{abstract}
A core challenge in program synthesis is \textbf{online library
learning}: the incremental acquisition of reusable abstractions under uncertainty about future task demands. Existing algorithms treat library learning as retrospective compression over a static task distribution, where the learned library is determined by the corpus of \textit{past} tasks. 
However, real-world learning domains are often non-stationary, with tasks arising from a generative process that evolves over time.
We propose and test the hypothesis that in non-stationary domains human library learning selects abstractions prospectively: targeting compression of \textit{future} tasks. We study this question using the \textit{Pattern Builder Task}, a visual program synthesis paradigm in which participants construct increasingly complex geometric patterns from a small set of
primitives, transformations, and custom \textit{helpers} that carry forward across trials. Using this task, we conduct two experiments with complementary latent curricula, designed to dissociate between behaviors consistent with prospective compression, and alternative library learning accounts. 
Using six computational models spanning online library learning strategies, we show that human abstraction behavior reflects sensitivity to latent, non-stationary structure in the task-generating process. This behavior is consistent with prospective compression, and cannot be captured by existing retrospective compression-based algorithms, or inductive biases modeled by LLM-based program synthesis.
\end{abstract}

\section{Introduction}

From building a cathedral to learning to play an instrument, people solve complex problems by decomposing them into reusable sub-tasks~\citep{gobet2001chunking,qin2025planning}. 
Library learning formalizes this idea in the context of program synthesis: 
given a corpus of programs, 
it iteratively discovers a library of reusable sub-programs, each chosen to maximally compress the programs seen so far
~\citep{ellis2023dreamcoder,liang2010learning}. 
This approach has been fruitful in modeling abstraction learning in human cognition across domains, including list functions~\citep{liang2010learning,rule2024symbolic}, cognitive maps~\citep{sharma2022mapi,kryven2025cognitive}, and hierarchical planning policies~\citep{correa2025exploring}.


However, many natural problem domains are dynamic, with tasks arising from a non-stationary generative process, rather than a static task distribution assumed by current library learning approaches~\citep{bowers2023top,ellis2023dreamcoder}.
Under such conditions, how should a learner decide which abstractions to form and reuse, given only the tasks seen so far (Fig.~\ref{fig:fig1})? 
Would algorithms based on retrospective compression naturally extend to such non-stationary settings?



\begin{figure}[t]
    \centering
    \includegraphics[width=0.8\textwidth]{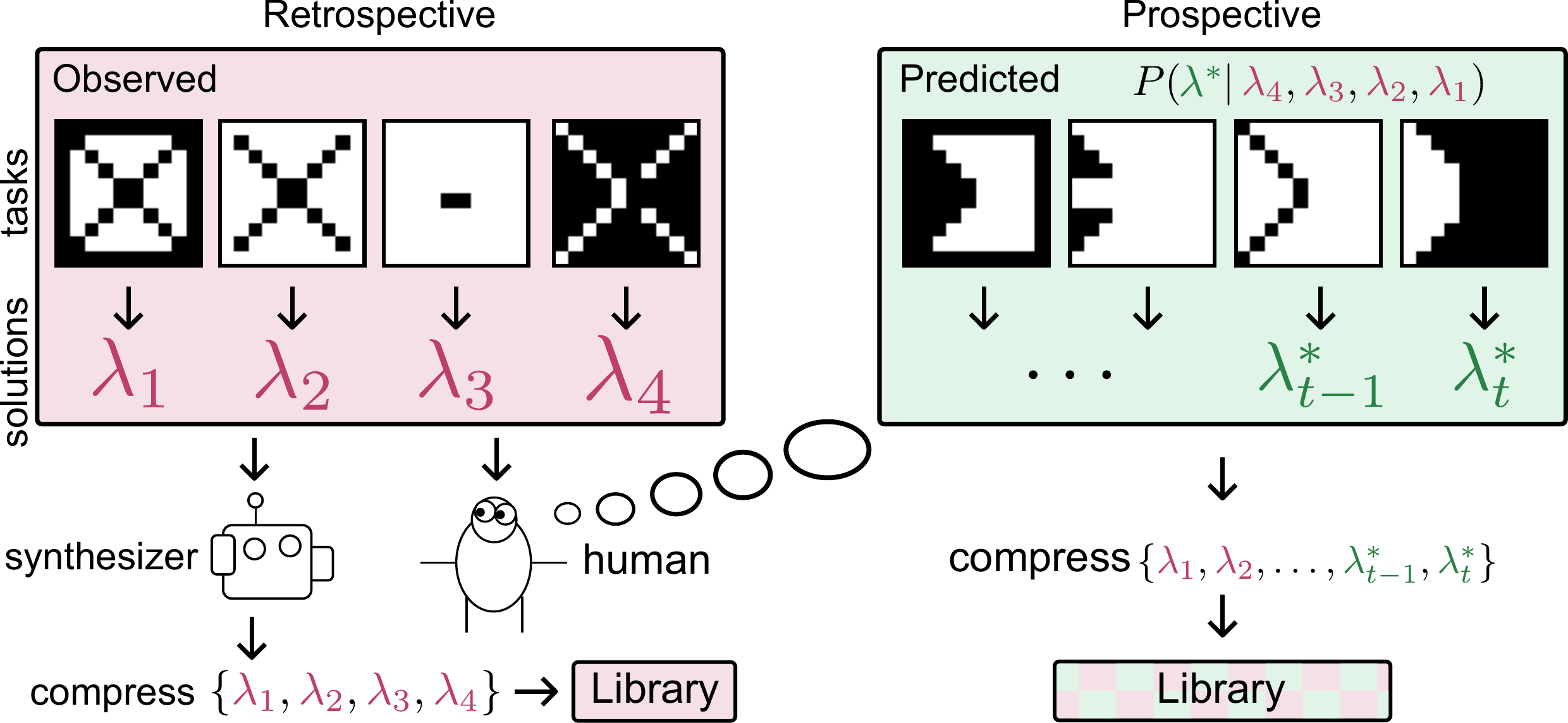}
    \caption{Existing library learning algorithms compress past tasks retrospectively (left). Human abstraction learning is better characterized as prospective compression over the task-generating process to select abstractions that transfer to future tasks (right).}
    \label{fig:fig1}
\end{figure}

In this work, we use human behavior as a diagnostic oracle \citep{lake2015human, acquaviva2022communicating,johnson2021fast} to identify computational principles that underpin online library learning, including under non-stationary conditions. To do this, we use the \textbf{Pattern Builder Task} (PBT), a visual program synthesis paradigm in which learners construct 
geometric patterns using a small set of primitive patterns and transformation operators (Fig.~\ref{fig:task_domain}).  
In PBT, 
solutions are expressed as programs composed of these primitives and transformations, 
and learners may store any intermediate construction as \textit{helpers}, which become available in subsequent trials. Figure \ref{fig:task_domain}(left) shows four examples of patterns constructed in PBT.

We formalize online library learning in PBT as the problem of incrementally constructing a set of helpers (reusable abstractions) while solving a sequence of tasks. 
We consider three computational principles:
(1) \textbf{Compression:} minimizing description length over 
solutions~\citep{bowers2023top,ellis2023dreamcoder},
(2) \textbf{Inductive bias:} structural priors over representations~\citep{kryven2024approximate,kumar2022using},
and
(3) \textbf{Generative inference:} modeling the non-stationary task distribution to anticipate future tasks~\citep{cogsci2026}.

We find that human abstraction decisions are sensitive to the latent generative structure of the task distribution, suggesting that human library learning is driven by prospective compression, and cannot be explained by retrospective compression or inductive bias alone.
In sum, we contribute: 
\begin{enumerate}
    \item a controlled experimental paradigm for online library 
    learning in which latent curricula dissociate prospective from retrospective compression strategies,
    \item 
    a family of computational models instantiating retrospective 
    compression and inductive bias, evaluated against human behavior, and
    \item the first demonstration that human online abstraction 
    learning is driven by prospective compression of the 
    task-generating process, and cannot be explained by retrospective 
    compression or inductive bias alone.
\end{enumerate}

\section{Library Learning}\label{sec:formalism}

We formalize pattern construction as inductive program synthesis within the programming-by-example (PBE) paradigm \cite{gulwani2017program}. A domain-specific language (DSL) defines a space of programs $p \in \mathcal{P}$ composed of geometric primitives $\mathcal{X}$ and transformation operators $\mathcal{T}$.  Let $c$ be a target pattern, to construct the target $c$ means to find a program $p$ such that $\llbracket p \rrbracket = c$.

A \textbf{library} $\mathcal{L}$ extends this base DSL with a finite set of helper programs $\mathcal{H} \subset \mathcal{P}$, which function as primitives in subsequent programs. Given a library $\mathcal{L}$, programs are represented as abstract syntax trees (ASTs), where leaves correspond to primitives $\mathcal{X}$ or helpers $\mathcal{H}$, and internal nodes correspond to transformation
operators $\mathcal{T}$.
Given library $\mathcal{L}$, target pattern $c$ can be solved by $\llbracket p' \rrbracket_{\mathcal{L}} = c$, where $\llbracket \cdot \rrbracket_{\mathcal{L}}$ denotes execution under the library-extended DSL. When $\mathcal{L} = \emptyset$,  $\mathcal{P}_{\mathcal{L}}$ reduces to the base DSL. 

Given a sequence of targets $\mathcal{C} = (c_1, \ldots, c_T)$ with solutions $\mathcal{P}' = \{p_t\}_{t=1}^{T} \subset \mathcal{P}$ where $\llbracket p_t \rrbracket = c_t$ for all $t$, the classic \textbf{library learning problem} is to find a set of helpers $\mathcal{H}^*$ maximizing compression utility~(CU):
\begin{equation}
\mathcal{H}^* = \arg\max_{\mathcal{H} \subset \mathcal{P}} \; 
    \mathrm{CU}(\mathcal{H}, \mathcal{P}'),
\qquad
        \mathrm{CU}(\mathcal{H}, \mathcal{P'}) = 
        \sum_{h \in \mathcal{H}} |h| \cdot \mathrm{occ}_\mathcal{H}(h, \mathcal{P'}),
        \label{eq:CU}
\end{equation}
where $|h|$ denotes the size of helper $h$ (measured in the number of DSL 
primitives in its AST), and $\mathrm{occ}(h, \mathcal{P'})$ denotes the 
number of programs in $\mathcal{P'}$ in which $h$ appears as a sub-program without being a subprogram of another program in $\mathcal{H}$ (this last requirement prevents double counting). Equation~\ref{eq:CU} follows the standard compression utility definition~\citep{bowers2023top, ellis2023dreamcoder}, and reflects the number of total AST nodes saved across all programs. 
In Appendix~\S\ref{si:proof}, we show that finding the optimal single helper is NP-complete even in a restricted flat-program setting.  

Current methods  construct library $\mathcal{L}$  by selecting helpers that maximize the retrospective CU (Eq.~\ref{eq:CU}) over programs seen so far \citep{ellis2023dreamcoder,bowers2023top}. 
In contrast, we hypothesize that human library learning may be solving a \textbf{prospective library learning problem}, which entails 
finding $\mathcal{H}$ that optimizes compression of the entire solution corpus $P^*$, $\mathrm{CU}(\mathcal{H}, \mathcal{P^*})$ (Figure \ref{fig:fig1}). 
Formally, let $\mathcal{C}_{1:t} = (c_1, \ldots, c_t)$ be the corpus of tasks observed 
up to time $t$, with solutions $\mathcal{P}'_{1:t} = \{p_s\}_{s=1}^{t}$,
the prospective library learning problem is to find:
\begin{equation}
    \mathcal{H}^* = \arg\max_{\mathcal{H} \subset \mathcal{P}} \; 
    \mathbb{E}_{P(P^* \mid \mathcal{C}_{1:t})} 
    \left[ \mathrm{CU}(\mathcal{H}, P^*) \right],
    \label{eq:prospective}
\end{equation}
where $P( P^*  \mid \mathcal{C}_{1:t})$ is a distribution over the full solution set $P^* = \{p_1, \ldots, p_{t+k}\}$ for the entire task corpus $\mathcal{C}_{1:t+k}$. 

We test the hypothesis that prospective library learning better captures how people form reusable abstractions by comparing human helper learning behavior in PBT to a family of computational models instantiating retrospective compression and inductive bias, and measure whether human abstraction selection exceeds what these accounts predict.

\begin{figure*}[t]
\centering
\includegraphics[width=\textwidth]{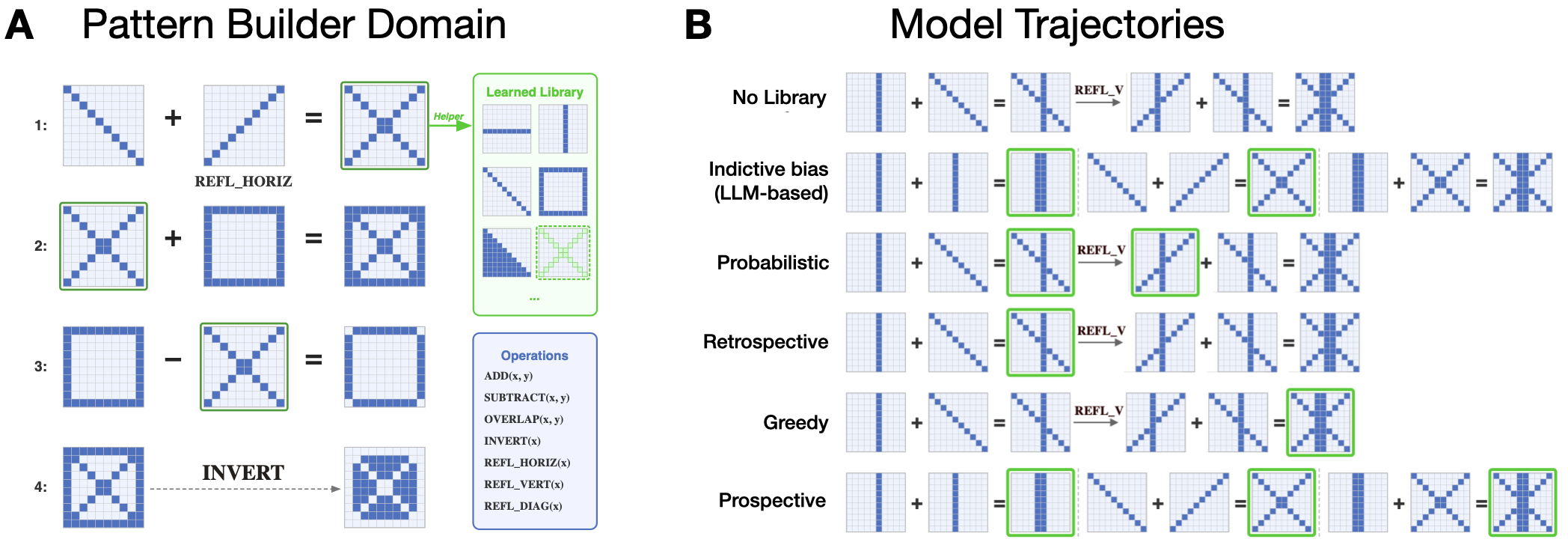}
\caption{ Example tasks in the \textit{The Pattern Builder}. 
\textbf{A} 
Learners are given an initial set of 5 geometric primitives (green box) and seven fixed operations (blue box). 
Primitives can be combined 
using operations, and any new patterns 
can be saved as `helpers' (e.g., the diagonal cross highlighted in green). Helpers carry over to subsequent tasks.
\textbf{B} Example helper-creating strategies across models. Green highlighting indicates abstractions (helpers) promoted to the learned library by each model.  Inductive biases in helper creation select for plausible regular geometric patterns (as modeled by LLM program synthesis). Probabilistic library learning stochastically promotes steps in MDL solutions.
Retrospective (DreamCoder-style) compression  promotes helpers that maximize compression of tasks observed in the past. Greedy compression always promotes the most recent solution, maximizing compression on the current trial.
Helpers promoted prospectively derive from predicting the distribution over future programs.}
\label{fig:task_domain}
\end{figure*}

\section{Computational Models}\label{sec:model}

We evaluate a family of models solving the program synthesis problem in PBT with or without library learning and evaluate which best captures human behavior.

\subsection{Baseline: Bottom-up Program Synthesis (No Library)}

This model performs exhaustive search 
over the base DSL, 
corresponding to the case where library $\mathcal{L} = \emptyset$, and   $\mathcal{P}_{\mathcal{L}}$ reduces to $\mathcal{P}$.  We perform a bottom-up enumeration over program size, constructing hypothesis classes $\mathcal{P}_k$ of programs with size $k$:
\[
\mathcal{P}_{k+1} = \{\, t(p_1,\ldots,p_n) \mid t \in \mathcal{T},\; p_i \in \mathcal{P}_{\le k} \,\}.
\]
Search proceeds in increasing order of $k$, corresponding to a breadth-first traversal of the program space. To reduce redundancy, we prune programs using observational equivalence \citep{albarghouthi2013recursive,udupa2013transit}, retaining a single shortest representative per equivalence class.  
This model solves each task independently via program enumeration, with only a Minimal Description Length (MDL) bias \cite{chater1999search,chater2003simplicity,rule2024symbolic}. 


\subsection{Library Learning with Retrospective Compression}

The second class of models augments bottom-up synthesis with incremental abstraction reuse over static task distributions. After solving each task $i$, program synthesis produces a derivation trace $\mathcal{D}^{(i)} = \{p^{(i)}_1, p^{(i)}_2, \ldots, p^{(i)}_{T_i}\}$,
where $p^{(i)}_{T_i} = p^{(i)}$ is the final solution program and intermediate elements correspond to subprograms constructed during search. We define an abstraction operator that selects a subset of candidate subprograms from the trace to promote to the library:
\[
\mathcal{L}^{(i+1)} = \mathcal{L}^{(i)} \cup \mathcal{A}(\mathcal{D}^{(i)}).
\]

Different instantiations of $\mathcal{A}$ correspond to different assumptions about how compression is performed over derivations.  We consider three variants:

\textbf{Model 2a: Retrospective Compression (RC)}
The abstraction operator follows the same incremental library construction principle as DreamCoder \citep{ellis2023dreamcoder}. 
After solving task $i$, each candidate in the derivation trace $\mathcal{D}^{(i)}$ is scored by its compression utility (Eq.~\ref{eq:CU}) against the current solution corpus $\{p_1, \ldots, p_i\}$. The top-$k$ helpers are selected greedily in descending order of utility (recomputed after adding each helper):
\[
\mathcal{A}_{\mathrm{RC}}(\mathcal{D}^{(i)}) = \operatorname*{arg\,max}_{
\mathcal{H} \subseteq \mathcal{D}^{(i)},\, |\mathcal{H}|=k} 
\mathrm{CU}(\mathcal{H},\, \{p_1, \ldots, p_i\}).
\]
Unlike DreamCoder, which adds a single helper per task, RC adds $k$ helpers per problem.
To better compare RC's compression utility against human participants, we set $k$ to be the mean number of helpers created per task in the behavioral experiments.

\textbf{Model 2b: Greedy Library Learning (GL)}
Model 2a is computationally expensive, and human cognition is often constrained by its computational resources \citep{lieder2020resource,zhao2024model}. We thus consider an abstraction operator that selects only the terminal program, corresponding to a local greedy MDL update over completed programs:
\[
\mathcal{A}_{\mathrm{GLL}}(\mathcal{D}^{(i)}) = \{p^{(i)}_{T_i}\}.
\]

\textbf{Model 2c: Probabilistic Library Learning (PL)}
We further consider an abstraction operator that stochastically selects elements from the derivation trace, inducing a noisy compression process:
\[
\mathcal{A}_{\mathrm{PLL}}(\mathcal{D}^{(i)}) = \{p \in \mathcal{D}^{(i)} \mid z_p = 1,\; z_p \sim \mathrm{Bernoulli}(q)\}.
\]

These three models implement different strategies for promoting helpers, based on different compression criteria over a static task distribution, i.e., without predicting whether a set of helpers might be useful for some unseen, future tasks.

\subsection{Library Learning with Inductive Bias Only (LLM-based)}

The third model family operationalizes inductive bias-driven abstraction learning without an explicit compression objective. We use a large language model (LLM) for program synthesis in a general-purpose language (Python). At each trial $t$, the model is given:
(i) a description of the PBT domain,
(ii) the current target pattern,
(iii) the DSL primitives encoded as Python functions, and
(iv) any previously introduced helper functions. On any trials where the target is not reconstructed correctly, LLM is prompted to refine its response. Any created helpers carry forward to subsequent refinement iterations. Refinement is repeated for up to 5 times.
Here, all models use \textsc{gpt5.2, low reasoning} as a LLM backend, chosen as the minimal version of the \textsc{gpt} family with a sufficient program synthesis capability to complete targets across experiments within 5 refinement steps. Full prompts are given in Appendix \S\ref{si:prompts}.

\textbf{Model 3a: Memoryless (LLM-PS)}
The model is prompted to construct each pattern with access only to the current DSL and helper set. This captures abstraction driven purely by pretrained inductive biases over program structure, without explicit conditioning on task history beyond available helpers.
\textbf{Model 3b: LLM-PS with history (LLM-PS-H)}
The model prompt additionally includes the full sequence of previous tasks and solutions. This allows helper construction to be conditioned on previous tasks to test whether this imparts an additional benefit to recovering shared task structure.

\subsection{Human Library Learning as Prospective Compression}

We define \textit{prospective compression} as a normative principle for abstraction promotion (Figure \ref{fig:fig1}). 
Following Eq.~\ref{eq:prospective}, 
the prospective abstraction operator selects helpers that maximize expected 
compression utility over both the present and anticipated future solutions: 
\[
\mathcal{A}_{\mathrm{PC}}(\mathcal{D}^{(i)}) = 
\operatorname*{arg\,max}_{\mathcal{H} \subseteq \mathcal{D}^{(i)}} \;
\mathbb{E}_{P(\mathcal{P}^* \mid \mathcal{C}_{1:i})}
\!\left[\mathrm{CU}(\mathcal{H},\, \mathcal{P}^*)\right],
\]
where $P(\mathcal{P}^* \mid \mathcal{C}_{1:i})$ is a distribution over the full (including future) solution corpora, conditioned on the tasks observed so far. This definition formalizes our hypothesis that human library learning instantiates prospective compression, and is deliberately agnostic of how $P(\mathcal{P}^* \mid \mathcal{C}_{1:i})$ is specified -- requiring only that helper selection compresses future tasks rather than just the past ones. 

\begin{figure*}[t]
\centering
\includegraphics[width=\textwidth]{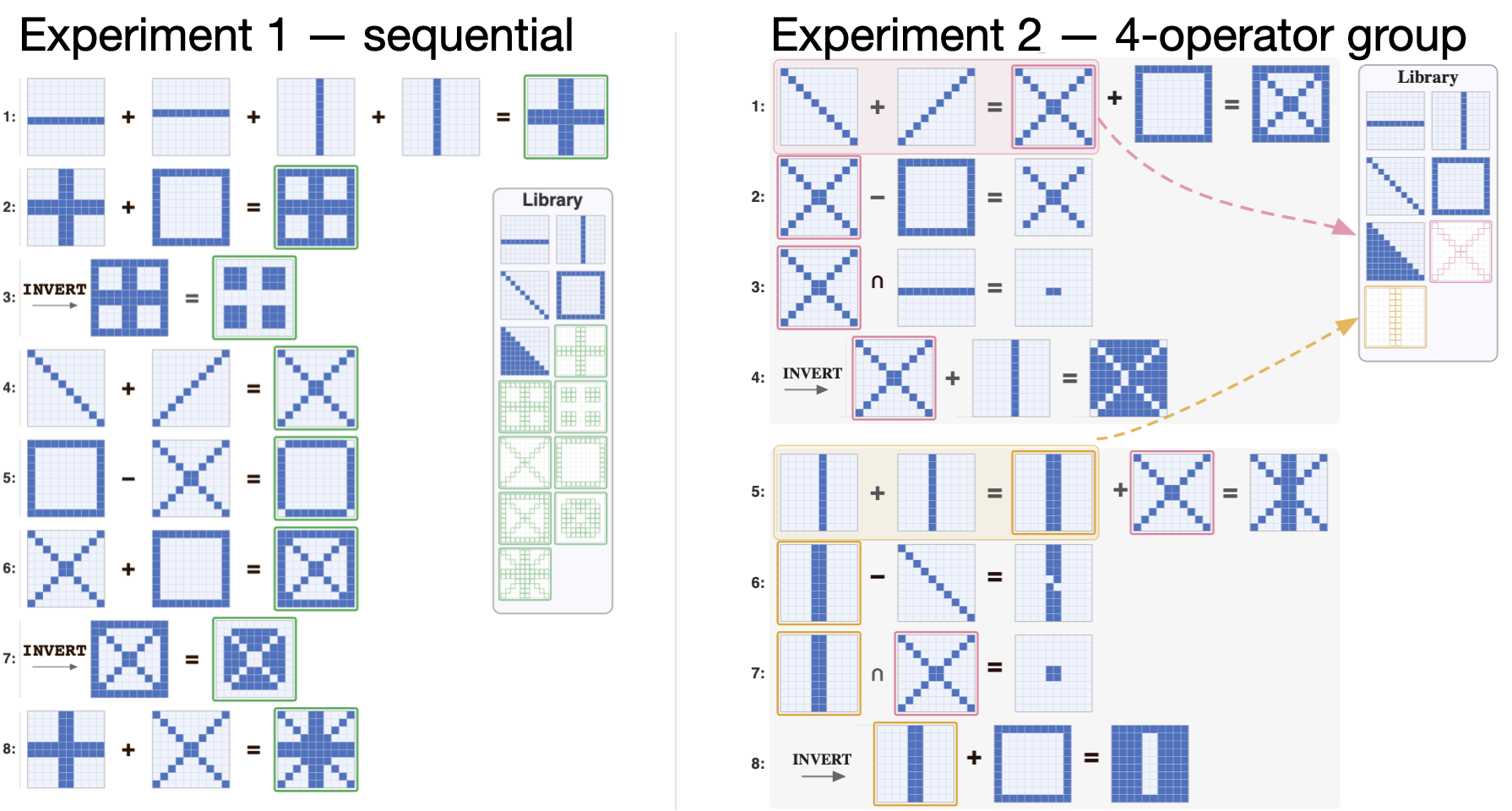}
\caption{Two complementary curricula in Experiments 1 and 2 illustrated by derivations of the first 8 targets from each experiment. Coloured highlights indicate helpers implied by each curriculum. In Experiment 2, gray background panels indicate common operator groups. Each group derives from a fixed set of four transformations (\textsc{+, -, $\cup$, ($\neg$, +)}) applied to a \textit{helper}, shared within the group, and another geometric primitive.}
\label{fig:curr}
\end{figure*}

Figure~\ref{fig:task_domain} illustrates divergent predictions of these models about which abstractions get promoted to the library on any given trial. Given a shared construction task, GL retains only the final solution program, RC selects abstractions to optimize compression of tasks seen so far, and LLM-based models promote helpers consistent with geometric inductive biases, and prospective compression retains helpers which it predicts to be shared by future tasks. 

\section{Experiments}

We evaluate which model best captures library learning strategies in human participants using two complementary latent curricula: a sequential curriculum (Experiment~1), under which retrospective compression is expected to be an effective strategy, and an operator-group curriculum (Experiment~2), designed to be tractable only under prospective compression of the task-generating distribution. In both experiments, we compare the corpus compression utility of human top-$k$ helpers 
against those selected by RC and the LLM-based models (as a proxy for abstractions driven by inductive biases).

\subsection{Curriculum Designs}

An ordered sequence of target patterns $(c_1, c_2, \ldots, c_T)$ with corresponding solution programs $(p_1, p_2, \ldots, p_T)$ forms a \textit{curriculum}. Curricula differ in their meta-structure $\mathcal{G}$: a generative process that imposes structured dependencies on solutions, 
either through preceding solutions or shared latent sub-programs. Figure~\ref{fig:curr} illustrates two possible structures, described below.

Experiment~1 tests a \textbf{sequential} curriculum, where each solution is obtained by applying a transformation $\tau \in \mathcal{T}$ to a preceding solution $p_{t-n}$ and primitive $x \in \mathcal{X}$:
\[
\mathcal{G}_{\mathrm{seq}}: \quad p_{t} = \tau(p_{t-n},\, x), 
\qquad \tau \in \mathcal{T},\ x \in \mathcal{X},\ 0 < n < t,
\]
This makes retrospective compression an effective strategy, as each target can be derived by incrementally extending a prior solution.

In contrast, Experiment~2 follows a generative structure designed to make prospective compression more effective. We say that a group of $K$ consecutive targets $\{c_t, \ldots, c_{t+K-1}\}$ follows a \textbf{shared helper} curriculum with respect to a latent sub-program $h \in \mathcal{P}$ 
if $h$ appears as a sub-program of every solution in the group:
\[
\mathcal{G}_{\mathrm{hlp}}: \quad h \sqsubseteq p_{t+k}, 
\qquad k = 0, \ldots, K-1,
\]
where $\sqsubseteq$ denotes the sub-program relation. A learner who recovers $h$ can represent all $K$ targets compactly, making future solutions predictable up to $K$ trials in advance.

A special case is the \textbf{$K$-operator-group} curriculum, where the $K$ solutions in each group are generated by applying a fixed set of operators to the shared sub-program $h$ and a primitive $x \in \mathcal{X}$. 
In Experiment~2, $K = 4$, giving:
\[
\mathcal{G}_{\mathrm{grp}}: \quad \{p_{t+k}\}_{k=0}^{3} = \bigl\{\,\textsc{Add}(h, x),\ \textsc{Subtract}(h, x),\ \textsc{Overlap}(h, x),\ \textsc{Add}(\textsc{Invert}(h), x)\,\bigr\},
\]
subject to the constraint that $h$ is a parsimonious shared structure: no target in the group can be derived from any other target in the group by a program shorter than that prescribed by $\mathcal{G}_{\mathrm{grp}}$. This constraint ensures that prospective compression dominates 
retrospective compression in Experiment~2, as we verify empirically in \S\ref{sec:results} using corpus compression utility (Eq.~\ref{eq:CU}).

Figure~\ref{fig:curr} shows the first 8 trials from each curriculum, illustrating the distinct helper promotion strategies induced by these designs. Detailed curriculum definitions for both experiments are given in Appendix~\S\ref{si:curriculum}.

\subsection{Human Experiments}

We tested a sequential curriculum in Experiment 1, and an operator-group curriculum in Experiment 2, with human participants (total $N=60$) recruited on Prolific Academic.
Participants used the Pattern Builder Task (PBT) interface to reconstruct $10\times10$ binary patterns by composing programs with a DSL of 6 primitives, 3 binary and 4 unary transformations (Figure~\ref{si:dsl}), with the option to promote any intermediate construction into a persistent helper reusable on subsequent trials (Figure~\ref{fig:task}).
In Experiment~1, $N{=}30$ participants (14 Female; age $36.2 \pm 9.8$) solved a sequential curriculum (14 target patterns, Fig.~\ref{si:e1}),
and in Experiment~2, $N{=}30$ participants (13 Female; age $34.0 \pm 11.0$) solved the operator-group curriculum (16 target patterns, Fig.~\ref{si:e2}).
Details of the behavioral experiments are available in Appendix~\S\ref{app:beh_exp}.

\subsection{Computational and Behavioral Metrics}

We analyze human and model performance using the following metrics:

\textbf{Number of steps.} The number of transformation operations (\textsc{Add, Subtract, Refl\_H, Refl\_V, Refl\_D, Invert, Intersect}) required to complete a pattern. We treat this as an \emph{upper bound} on efficiency for human data, since participants may take exploratory actions or make mistakes, inflating step counts independently of library quality. Moreover, possessing a high-quality library does not guarantee shorter programs, as people tend to deviate from optimal search~\cite{kryven2024approximate}. We therefore use step count only as a coarse efficiency proxy.

\textbf{Library size.} The number of helpers saved in PBT (for humans), or the number of sub-programs (Python functions or sub-ASTs) abstracted by LLM-based models and RC.

\textbf{Top-$k$ helpers ($h^k$).} For deterministic models (i.e., RC), we select the top-$k$ AST sub-trees ranked by compression utility, where $k$ is set to the empirical mean human library size. For human and LLM-based models we take the top-$k$ most frequently created helpers at each trial, where $k$ is the empirical mean library size. The rationale is that helper creation is driven partly by a common generative process and partly by noise (e.g., interface exploration, mistakes). Aggregating across participants and selecting the top-$k$ by frequency isolates the shared signal from individual variation.

\textbf{Corpus compression.} The compression utility of a helper subset $h^k \subseteq \mathcal{H}$ evaluated against the full (held-out) program corpus $\mathcal{C}^*$. This is our \emph{primary} metric for detecting prospective abstraction in human library learning.

\textbf{Oracle helpers.} Helpers derived from ground-truth sub-ASTs up to trial $t$, ranked by compression utility over the full corpus. These represent the helpers RC would select given hindsight, and serve as an upper-bound reference for library quality.

\subsection{Results}\label{sec:results}

\paragraph{PBT is tractable to human participants}
Across both experiments, pattern completion success rate was high (E1: $M = 92.4\%$, $SD = 14.3\%$, 95\% CI $[87.0\%, 97.7\%]$; 
E2: $M = 81.9\%$, $SD = 16.5\%$, 95\% CI $[75.7\%, 88.0\%]$; 
$n = 30$ each). Detailed accuracy plots are given in Appendix~\S\ref{si:behavior}. This result confirms that PBT is tractable for human participants, and motivates filtering out any models that fail to complete the full task sequence. Three models fail on this criterion. The No Library baseline exhausts its computational budget on trial 7 in Experiment~1 and trial 9 in Experiment~2, consistent with the super-exponential growth of the search space established in Appendix~\S\ref{si:search_size}. GL completes the Experiment~1 curriculum but fails from trial 9 onward in Experiment~2, where the operator-group structure requires helpers that greedy local updates do not produce. PL fails on the final four trials of Experiment~1. The remaining three models RC, LLM-PS, and LLM-PS-H complete both task sequences and are compared with human behavior below.

\begin{figure*}[t]
\centering
\includegraphics[width=\textwidth]{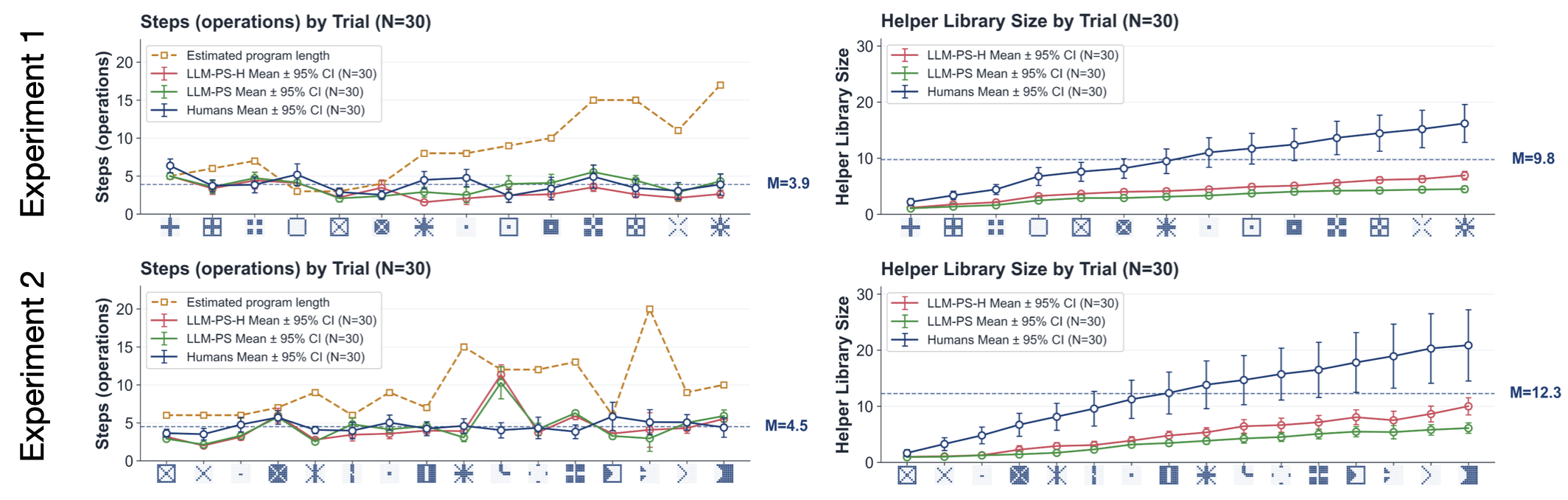}
\caption{(Left) Human solution efficiency, measured in steps to complete a task per trial, compared to estimated program length (AST operation count), and compared to models solution efficiency, across  trials. Raw program length estimate was obtained by expanding the shortest program found by RC into their constituent DSL primitives.  
(Right) The average accumulated size of participants' helper libraries, compared to models' learned libraries, across trials. LLM-based models maintain smaller libraries compared to humans, while mimicking the library growth trend. }
\label{fig:results_behaviour}
\end{figure*}

\paragraph{People use helper libraries to complete tasks in both experiments.}
Figure~\ref{fig:results_behaviour} shows the number of new helpers added by participants per trial, across experiments 
(E1: $M = 1.16$, $SD = 0.67$, 95\% CI $[0.91, 1.41]$;
E2: $M = 1.30$, $SD = 1.11$, 95\% CI $[0.89, 1.72]$). The vast majority of participants created at least
one helper (E1: $28/30$; E2: $30/30$;), confirming that
nearly everyone engaged in library learning. 
The size of participants' helper libraries increased over the course of each experiment (Figure~\ref{fig:results_behaviour}). While library contents varied between individuals (Appendix Figures \ref{si:e1_helpers} and \ref{si:e2_helpers}), there was a significant overlap between individual helper libraries, consistent with helper strategies implied by our experiment design (Figure~\ref{fig:results2}~A).

\paragraph{Library learning reduces search complexity.}
To analyze human number of steps, we exclude trials on which the target was built incorrectly (3\% of total trials).
Across both experiments, people consistently solved tasks in reliably fewer steps than solutions expressed in raw primitives alone (two-sided paired $t$-test aggregated by trial; E1: $M_{\text{human}}=3.93$ vs.\ $M_{\text{prog}}=8.64$, $t(13)=-3.72$, $p=2.6\times10^{-3}$, Cohen's $d_z=-1.00$; E2: $M_{\text{human}}=4.51$ vs.\ $M_{\text{prog}}=9.56$, $t(15)=-5.04$, $p=1.5\times10^{-4}$, $d_z=-1.26$; Figure~\ref{fig:results_behaviour}). 
Further, the length of human-built programs remained roughly flat across trials ($\sim$3--5 steps; E1: slope $=-0.09$ steps/trial, $p=0.22$; E2: slope $=+0.05$, $p=0.19$)
as the raw-primitive program length grew significantly over trial (E1: slope $=+0.92$ ops/trial, $R^2=0.72$, $p=1.4\times10^{-4}$; E2: slope $=+0.47$, $R^2=0.31$, $p=0.025$), consistent with library learning serving to maintain bounded solution depth.


\paragraph{Human abstraction learning tracks prospective compression}
Figure~\ref{fig:results2} shows that in Experiment~1, RC closely 
approximates oracle compression, confirming that its sequential curriculum can be effectively solved by retrospective compression. Human corpus compression similarly approximates oracle helpers in Experiment~1, consistent with both retrospective and prospective accounts under this curriculum.
In contrast, in Experiment~2, where the operator-group curriculum favours prospective compression, corpus compression of human helpers consistently exceeds both RC and LLM-based models, closely tracking oracle compression. This result suggests that retrospective compression alone is insufficient 
to explain human abstraction selection.

LLM-PS and LLM-PS-H recover a subset of geometrically canonical helpers that partly overlap with human helpers (Figure~\ref{fig:results2}), suggesting that inductive bias may contribute to abstraction selection. However, human corpus compression consistently outperforms both LLM-based models across both experiments, indicating that inductive bias alone---as captured by LLM-based synthesis---cannot account for human abstraction behaviour regardless of curriculum structure.

\begin{figure*}[t]
\centering
\includegraphics[width=\textwidth]{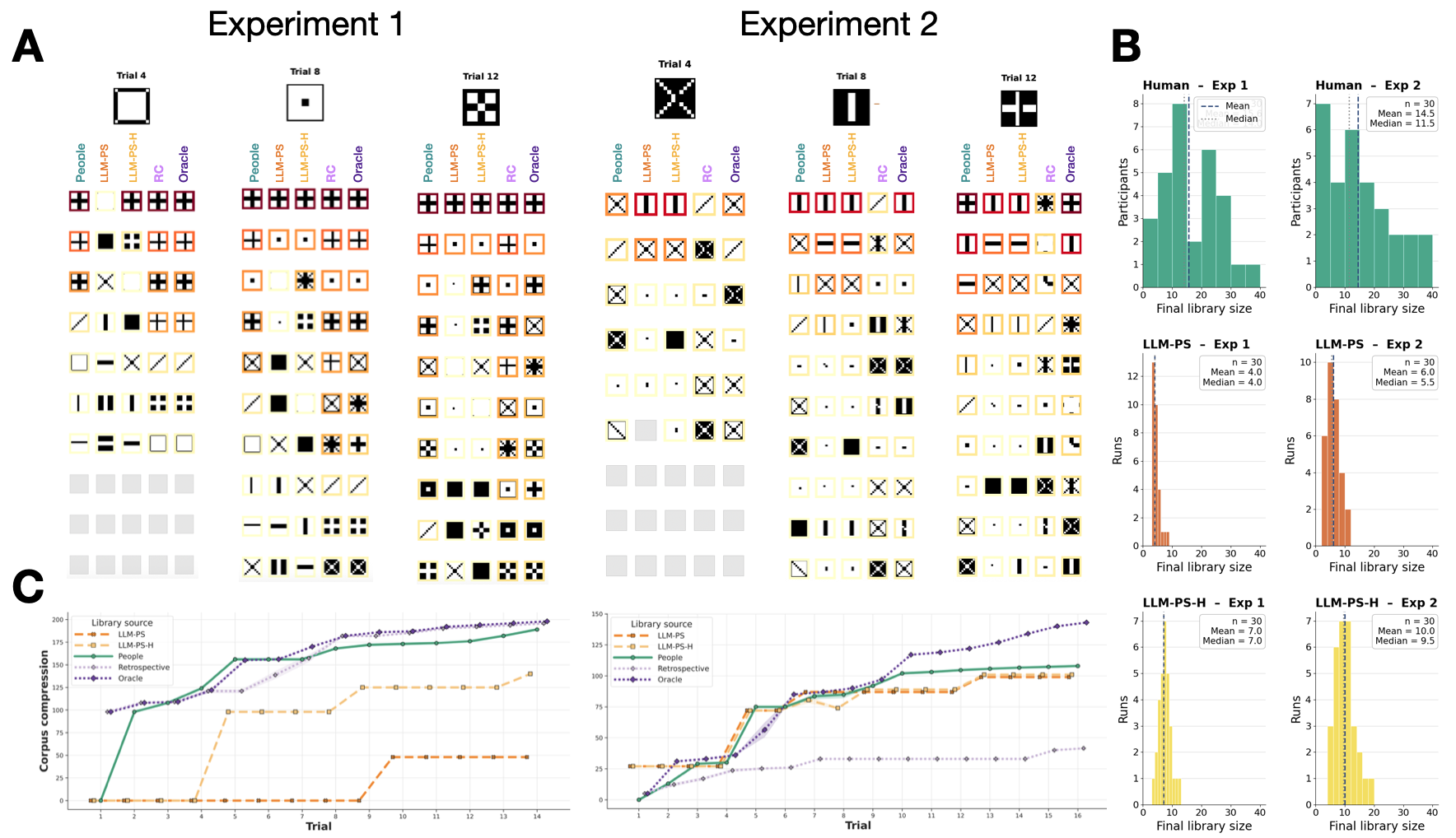}
\caption{Libraries created by people and models.
\textbf{A}. The evolution of human and model top-k helpers over trials, where $k$ is set to mean human library size at a given trial.  \textbf{B}.  Histograms of learned library sizes for people and LLM-based models in both experiments. 
\textbf{C}. Corpus Compression of top-k human helpers compared to models. 
Compression values reflect means computed via randomized  tie-breaking whenever several helpers had equal rank. 
}
\label{fig:results2}
\end{figure*}

\section{Related Work}

\textbf{Library learning in program synthesis}
A central approach to abstraction learning in program synthesis is library learning via compression~\cite{liang2010learning,allamanis2018mining,shin2019program,grand2024lilo}. Systems such as DreamCoder \citep{ellis2023dreamcoder} iteratively compress solution programs into reusable abstractions using a wake--sleep procedure, while recent work such as Stitch \citep{bowers2023top} formulates library induction as greedy MDL optimization over program corpora. \citet{ren2026library} apply a similar compression-based approach to jazz harmonic patterns, integrating deductive parsing with library learning on e-graphs to discover reusable harmonic abstractions from a corpus. These approaches demonstrate strong performance in offline settings with fixed task distributions, but assume a static corpus and rely on computationally expensive global optimization. In contrast, we study \emph{online} library learning, where abstractions must be formed incrementally under uncertainty about future tasks.

\textbf{Behavioral studies of program induction and abstraction}
A growing body of work in cognitive science has used program synthesis frameworks to model human behavior in compositional tasks. \citet{tian2020learning} showed that DreamCoder-style library learning can partially account for human abstraction reuse, though additional biases (e.g., motor efficiency) were required to match behavior. More recent studies have examined abstraction learning in domains such as symbolic list processing \citep{rule2024symbolic,ham2025teaching}, hierarchical planning \citep{qin2025planning, correa2025exploring}, and cognitive maps \citep{sharma2022mapi, kryven2025cognitive}, suggesting that human abstractions reflect structured inductive biases beyond pure compression \citep{he2025bootstrapping}. Our work extends this line by studying abstraction learning \emph{across tasks} in an online setting where the task distribution itself evolves.

\textbf{Inductive biases and neural program synthesis}
LLMs trained on large corpora of code and natural language have demonstrated strong program synthesis capabilities~\citep{chen2021evaluating, austin2021program}, including in visual inductive reasoning tasks~\citep{acquaviva2022communicating, wang2023hypothesis}, and have been proposed as implicit models
of human inductive biases acquired from pretraining~\citep{binz2025foundation,
kryven2025cognitive}. This perspective motivates using LLM-based
synthesis as a computational model of inductive-bias
abstraction in the absence of history (LLM-PS).


\section{Discussion}

We provide converging behavioral and computational evidence that human abstraction learning is better characterized by prospective compression of the observed task corpus, rather than as driven by retrospective compression and inductive biases alone. Across both experiments, our results suggest that human behavior is consistent with inference over latent generative structure of the observed task distribution, which allows people to adapt abstraction choices to future tasks. 
 
\textbf{Implications for program synthesis and cognitive modeling}
For program synthesis, our results suggest that library learning algorithms should move beyond retrospective compression objectives and explicitly model non-stationary task distributions. 
One approach to making abstraction selection sensitive to dependencies between tasks could be implemented via explicit structure tracking over task sequences, for instance using program induction over the generative processes. 
For cognitive science, the results support the view that abstraction learning is tightly coupled to anticipatory structure modeling, where learners not only compress experience, but also infer how task structure evolves over time. This connects library learning to broader accounts of hierarchical prediction and sequence learning in cognitive science \citep{correa2025exploring,kryven2025cognitive, qin2025planning,sharma2022mapi,rule2024symbolic}.

\textbf{Limitations and future work}
Our study focuses on a constrained visual program synthesis domain, which allows us to precisely control experiment design, and analyze behavior in depth under different modeling assumptions. Future work should consider modeling prospective abstraction learning across other domains (i.e. number sequences, communication).
Further, our work focuses on empirically demonstrating  prospective abstraction learning in humans, without implementing a fully generative-process learner that explicitly infers the latent curriculum. A natural next step is to formally model the probability distribution of future tasks and solutions, and to perform joint inference over both the library and the latent task-generating process. 

\textbf{PBT as a candidate benchmark for prospective abstraction}
Beyond its role as a behavioral paradigm, PBT suggests a potential framework for evaluating abstraction learning algorithms. Solving individual PBT tasks does not require online library learning in principle --- a sufficiently strong synthesizer can solve each trial independently via brute-force search or direct generation. However, \emph{matching human abstraction behavior} requires selecting reusable structure that is not necessarily optimal for the current task, but supports efficient transfer across future, structurally related tasks. This distinction operationalizes sensitivity to latent task-generating structure, rather than task-level performance alone. While existing benchmarks such as ARC-AGI~\citep{chollet2019measure} and related sequential or adaptive evaluation suites primarily probe within-task compositional reasoning or episodic generalization, PBT introduces a setting where structure must be inferred and exploited across a temporally evolving task distribution. As such, PBT provides a basis for developing future benchmarks of prospective abstraction, where the central evaluation target is the ability to infer and reuse structure in non-stationary domains.

\section*{Author Contributions}
\textbf{M.K.:} Conceptualization, Methodology, Formal Analysis, Visualization, Writing -- Original Draft, Review \& Editing
Project Administration, Funding Acquisition, Supervision.
\textbf{B.Z.:} Conceptualization, Methodology, Funding Acquisition, Supervision, Writing -- Review \& Editing
\textbf{L.C.:} Software (Compression-based models), Formal Analysis, Writing -- Review \& Editing, Investigation, Visualization
\textbf{I.Z. L.A.:} Software (LLM-based models), Investigation, Visualization
\textbf{P.Z.} Software (human experiment), Formal Analysis (behaviour), Visualization, Writing -- Review \& Editing
\textbf{M.M.:} Formal Analysis (complexity analysis), Investigation, Visualization
\textbf{Y.P.:} Writing -- Review, Visualization
\textbf{E.S.:} Formal Analysis (hardness proofs), Writing -- Review \& Editing

\section{Acknowledgments}
E.S.\ receives funding from the European Research Council (ERC) under the Horizon Europe research and innovation programme (MSCA-GF grant agreement No.\ 101149800). M. K. was supported by NSERC Discovery grant RGPIN-04045-25.





\bibliographystyle{plainnat}

{
\small
\bibliography{refs}
}


\clearpage

\appendix


\section{Hardness of Library Learning}\label{si:proof}
\input{tmp}

\section{Pattern Builder Task}\label{si:task}

\subsection{Domain}
Targets are binary patterns defined on a $10\times10$ grid, where each
cell is either filled or empty. In each trial, subjects and models must
reproduce the target by constructing a program in a fixed
domain-specific language.

\subsection{Domain-Specific Language}
The DSL (Figure~\ref{si:dsl}) consists of six geometric primitives
($\mathcal{X}$): \texttt{blank}, \texttt{line\_horizontal},
\texttt{line\_vertical}, \texttt{diagonal}, \texttt{square} (border
only), and \texttt{triangle}; three binary operators
($\mathcal{T}_{\mathrm{bin}}$): \texttt{add}, \texttt{subtract}, and
\texttt{overlap}; and four unary operators
($\mathcal{T}_{\mathrm{un}}$): \texttt{invert},
\texttt{reflect\_horizontal}, \texttt{reflect\_vertical}, and
\texttt{reflect\_diag}. Programs are represented as trees whose leaves
are primitives or previously defined helpers and whose internal nodes
are transformation operators. Executing a program yields a $10\times10$
binary pattern, and a trial is scored as correct when this output exactly
matches the target.

\subsection{Helper Mechanism}
At any point in a trial, a subject can save the output of a step as a
\emph{helper}. From then on, the helper behaves as a new primitive: it
appears in the helper panel and can be fed to any operator, including
operators that build further helpers. Each helper is shown as a
thumbnail of its pattern; helpers are not named, and identical patterns
are stored only once. The library is private to each subject and
persists across the rest of the main task and the free-play phase.
Subjects can delete helpers at any time. Helpers built during the
tutorial are cleared before the main task starts.

\subsection{Interface}
Each trial shows the target (top-left), a workspace canvas where the
current program is rendered (center), a step list with one line per
program step (right), and the helper library (bottom; see
Figure~\ref{fig:task}). A subject builds the program by picking a
primitive or helper and applying an operator. Each step adds a line to
the program and updates the canvas. Before a step is committed, the
canvas previews its result and the subject can cancel it. Once
committed, a step cannot be removed. Submit is available at any time.

\subsection{Trial Flow}
A trial began with the target shown next to a blank canvas. Subjects
built up a program one step at a time. They could preview each step
before committing it, cancel an uncommitted step, save a step as a
helper, or delete a helper. Committed steps could not be undone, and
there was no time limit. Pressing Submit ended the trial: an exact
match with the target scored 1, anything else scored 0. A feedback
panel then showed the outcome, the number of steps used, and the
cumulative score, after which the next trial began with a fresh
canvas. All subjects saw trials in the same order.

\subsection{Tutorial and Comprehension Check}
Subjects first worked through an interactive tutorial that introduced
the canvas, primitives, operators, and helpers through guided practice.
They then took a four-item multiple-choice quiz covering the goal of a
trial, the scoring rule, helper persistence across trials, and the
preview feature. Any wrong answer sent the subject back to the
instructions to retake the quiz before continuing.

\subsection{Free Play and Debrief}
After the last curriculum trial, subjects entered a free-play block
with no target (Figure \ref{si:freeplay}), in which they could build whatever they liked using the
same DSL and helper library. Creations could be named and submitted to
a gallery. Subjects then completed a debrief questionnaire on
age, gender, perceived difficulty, engagement, helper usefulness, and
free-play enjoyment.

\subsection{Implementation}
The experiment is a static web application (HTML and JavaScript) and
runs in any modern desktop browser. Every action---steps, helper
saves, deletions, submissions, and their timestamps---is logged in the
browser and uploaded to the server at the end of each phase.

\begin{figure*}[t]
\centering
\includegraphics[width=\textwidth]{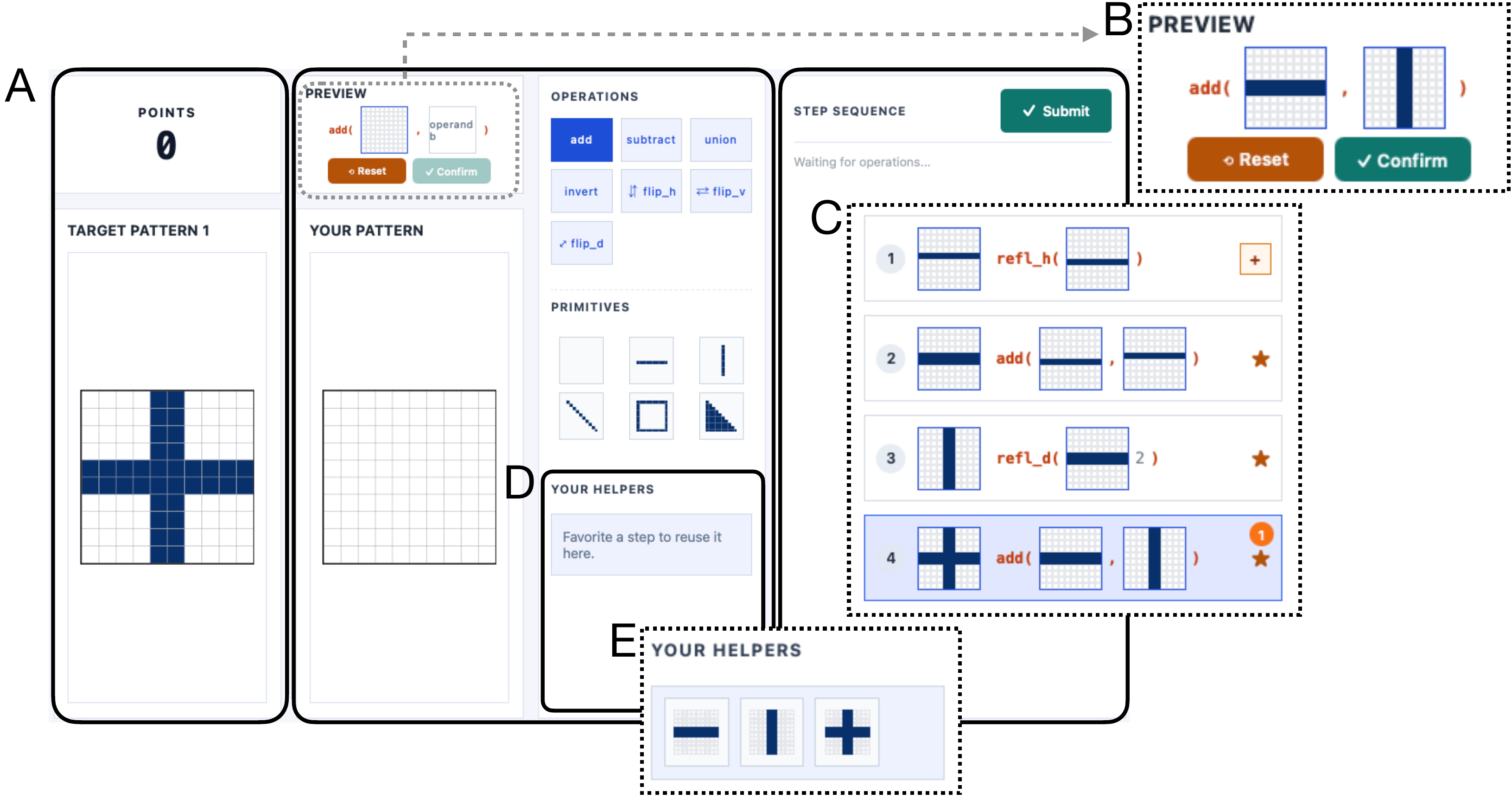}
\caption{Task interface. 
A. Starting view of a trial. Left: Total points so far and the target shape for this trial. 
Middle: Work space, matching the target shape with the provided Operations and Primitive shapes. 
Right: List of steps, each step corresponds to one line of program.
Black borders are only added in the paper for illustration.
B. Example preview.
C. Example programs. Each line has a line number, a thumbnail of the pattern that line creates, and the corresponding program (operations over primitives, steps, or helpers).
D. The initial helper space.
E. Helper space with example helpers.}
\label{fig:task}
\end{figure*}

\begin{figure}[ht]
  \centering
  \includegraphics[width=0.8\textwidth, height=0.45\textheight,keepaspectratio]{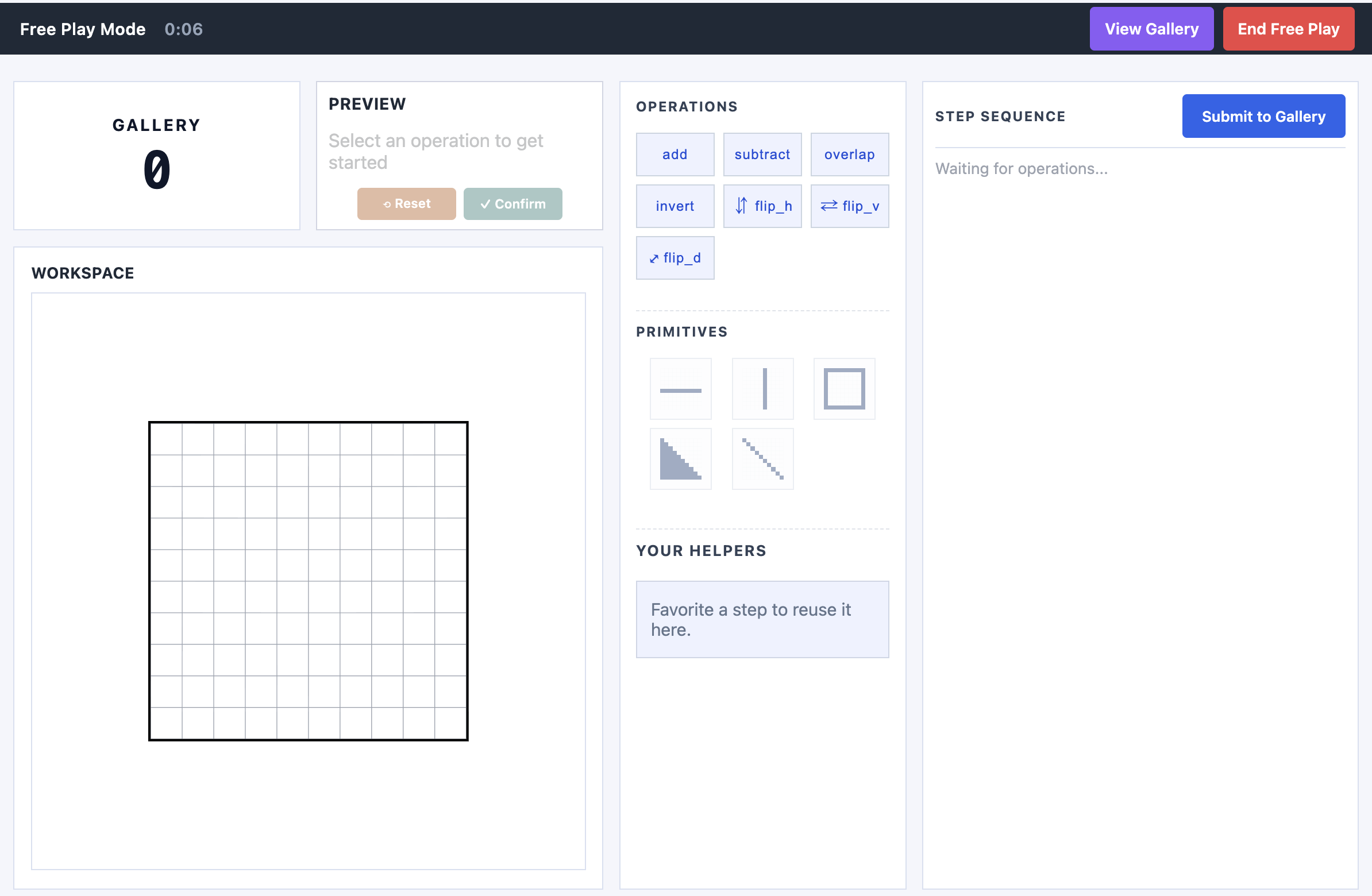}
  \caption{Free-play interface. After the final curriculum trial,
  subjects build any pattern they like using the same DSL and helper
  library, with no target. Creations can be named and submitted to a
  shared gallery.}
  \label{si:freeplay}
\end{figure}

\clearpage

\section{Behavioral Experiments}
\label{app:beh_exp}

\subsection{Pattern Builder Task}
In the Pattern Builder Task (PBT), participants reconstruct $10\times10$ binary patterns by composing programs from a shared domain-specific language (DSL) consisting of 6 primitives, 3 binary operations, and 4 unary transformations (Figure~\ref{si:dsl}). A key feature of the task is that any intermediate construction can be promoted to a persistent helper function, which then becomes reusable across all subsequent trials (Figure~\ref{fig:task}). Both experiments use this same task and DSL; they differ only in the latent generative structure of the target sequence. A full description of the PBT task and interface is provided in Appendix~\S\ref{si:task}.

\subsection{Participants}
Participants were recruited through Prolific Academic and gave informed consent prior to taking part. The study was approved by our institution's ethics committee (reference omitted for anonymity).
In Experiment~1, 30 participants were recruited (14 Female; mean age $36.2 \pm 9.8$ years), of whom 4 were excluded for task disengagement, leaving 26 in the analysis. Participants were compensated at \pounds7.27/hr. In Experiment~2, 30 participants were recruited (13 Female, 1 nonbinary; mean age $34.0 \pm 11.0$ years), with 5 excluded under the same criteria, leaving 25 in the analysis. Compensation was \pounds6.00/hr.

\subsection{Stimuli}
The two experiments used different target sequences, each designed to instantiate a distinct latent structure. In Experiment~1, the 14 targets in Set~1 (Figure~\ref{si:e1}) follow an auto-regressive structure: each target is a DSL transformation of the preceding one, with two long-range dependencies threading through the sequence. In Experiment~2, the 16 targets in Set~2 (Figure~\ref{si:e2}) follow an operator-group structure, in which targets are organized into groups of four that share a latent helper combined with a fixed set of four transformations; consecutive trials are designed to minimize surface similarity, encouraging discovery of the underlying group structure rather than shallow pattern matching. The formal curriculum design is described in Appendix~\S\ref{si:curriculum}.

\subsection{Procedure}
The experiment was conducted entirely in a web browser on participants' own desktop or laptop computers. Each session began with an interactive tutorial followed by a comprehension check to ensure familiarity with the interface and DSL. Participants then worked through the curriculum targets in a fixed order, constructing a solution program step-by-step on a canvas displayed alongside the target pattern. No evaluative feedback was given at any point, and there was no time limit. Following the final trial, participants completed a short free-play block and a debrief questionnaire. 

\clearpage

\section{Behavioral Results}\label{si:behavior}

\subsection{Overall Accuracy across Trials}
\label{si:accuracy}

\begin{figure}[ht]
  \centering
  \includegraphics[width=0.95\textwidth]{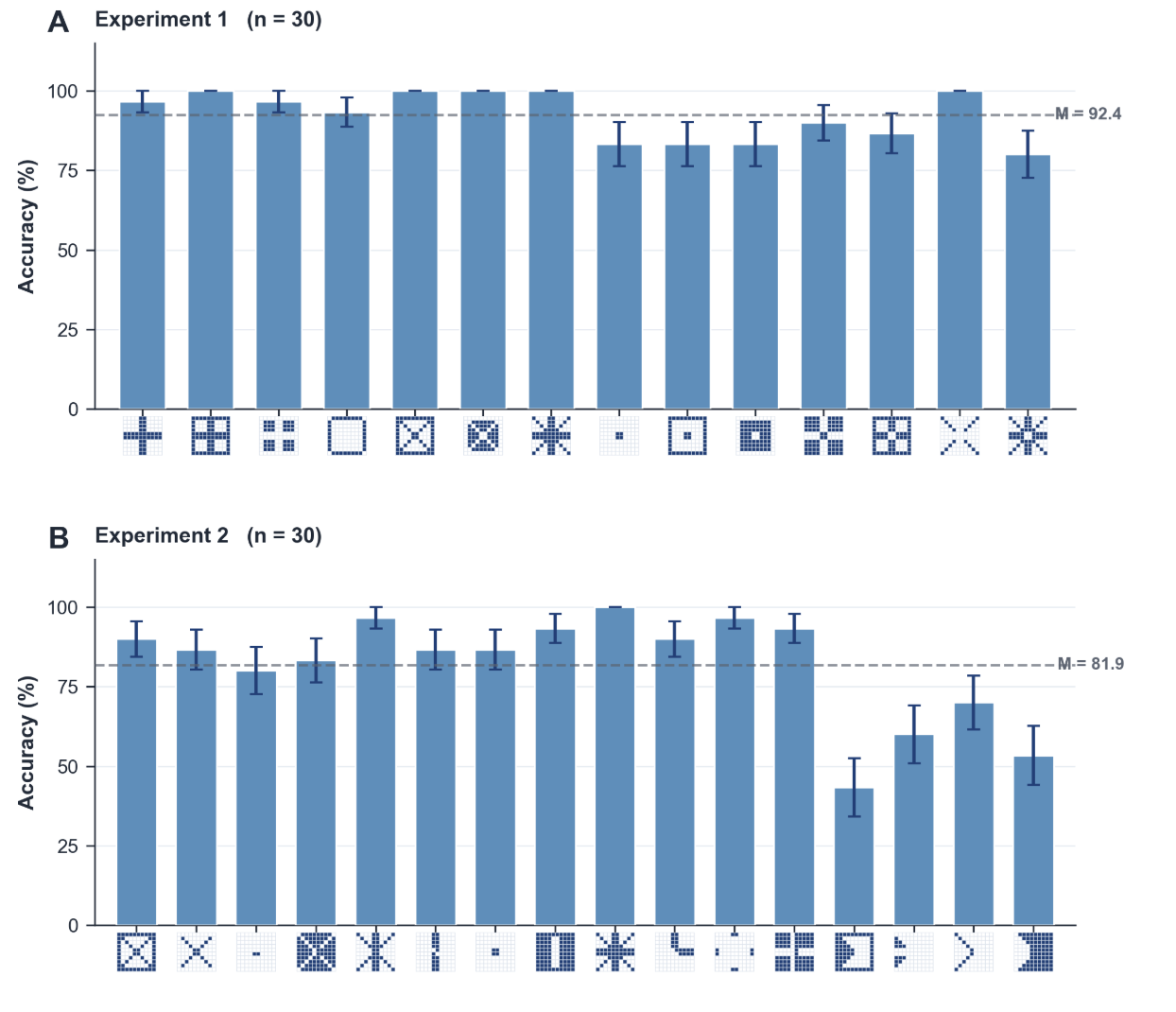}
  \caption{Mean human accuracy by trial in E1 (\textbf{A}) and E2
  (\textbf{B}); $n = 30$ each. Error bars: $\pm 1$ SEM. Dashed line:
  overall mean ($M$). Thumbnails below each bar show the 
  target pattern. }
  \label{si:fig-acc}
\end{figure}





\clearpage

\section{LLM Prompts}\label{si:prompts}

\subsection{Model 3b: LLM-PS with history (LLM-PS-H)}

The main user prompt outlines the task the LLM must follow, specifying what information is provided at each trial and the strict set of rules it must adhere to. This is the history-enabled version of the prompt and therefore includes additional context indicating that the task is part of a sequence of trials. The prompt defines the primitives and transformation functions that the LLM must use as building blocks to reconstruct the target composites. Primitives are represented as 10×10 binary arrays, where 0 denotes an empty cell and 1 denotes a filled cell. Transformations are provided as Python functions that operate on NumPy arrays. In addition, the prompt includes several illustrative examples demonstrating how to use primitives, transformations, and helper functions. These examples show how helper functions can be defined and reused, and each includes a brief description, a target 10×10 array, and a corresponding solution. It's important to note that the example target arrays are distinct from those used in the experimental trials.

\begin{verbatim}
You will be given a sequence of tasks where you will write functions to
produce 10x10 binary arrays.

The tasks will be given one trial at a time, along with the history of
tasks you completed in previous trials.

On each trial, you will be given:
(1) a target 10x10 binary array,
(2) a set of geometric PRIMITIVES as 10x10 binary arrays, and
    TRANSFORMATION operations as Python functions, and
(3) any helper functions you wrote to complete the tasks on previous
    trials.

The PRIMITIVES and TRANSFORMATIONS are stable and shared across trials.
Helper functions carry forward to all subsequent trials and may be reused.

On each trial, you will be given Python starter code with gaps to fill in.

CRITICAL RULES:
1. You must NOT create new primitives, hardcode any array elements in the
   output, redefine any provided variables or functions --- always call
   PRIMITIVES and TRANSFORMATION FUNCTIONS directly.
2. You may NOT use loops, list comprehensions, or import anything.
3. All reconstructions MUST be done using the provided PRIMITIVES,
   TRANSFORMATIONS, and helpers only.
4. Helpers must derive entirely from the provided PRIMITIVES,
   TRANSFORMATIONS and prior helpers.
5. Write Python code only. Do not include any comments or explanations.

PRIMITIVES:

blank = [
    [0, 0, 0, 0, 0, 0, 0, 0, 0, 0],
    ...  (10x10 zeros)
]

line_horizontal = [
    [0, 0, 0, 0, 0, 0, 0, 0, 0, 0],
    [0, 0, 0, 0, 0, 0, 0, 0, 0, 0],
    [0, 0, 0, 0, 0, 0, 0, 0, 0, 0],
    [0, 0, 0, 0, 0, 0, 0, 0, 0, 0],
    [0, 0, 0, 0, 0, 0, 0, 0, 0, 0],
    [1, 1, 1, 1, 1, 1, 1, 1, 1, 1],
    [0, 0, 0, 0, 0, 0, 0, 0, 0, 0],
    [0, 0, 0, 0, 0, 0, 0, 0, 0, 0],
    [0, 0, 0, 0, 0, 0, 0, 0, 0, 0],
    [0, 0, 0, 0, 0, 0, 0, 0, 0, 0]
]

line_vertical = [
    [0, 0, 0, 0, 0, 1, 0, 0, 0, 0],
    [0, 0, 0, 0, 0, 1, 0, 0, 0, 0],
    [0, 0, 0, 0, 0, 1, 0, 0, 0, 0],
    [0, 0, 0, 0, 0, 1, 0, 0, 0, 0],
    [0, 0, 0, 0, 0, 1, 0, 0, 0, 0],
    [0, 0, 0, 0, 0, 1, 0, 0, 0, 0],
    [0, 0, 0, 0, 0, 1, 0, 0, 0, 0],
    [0, 0, 0, 0, 0, 1, 0, 0, 0, 0],
    [0, 0, 0, 0, 0, 1, 0, 0, 0, 0],
    [0, 0, 0, 0, 0, 1, 0, 0, 0, 0]
]

diagonal = [
    [1, 0, 0, 0, 0, 0, 0, 0, 0, 0],
    [0, 1, 0, 0, 0, 0, 0, 0, 0, 0],
    [0, 0, 1, 0, 0, 0, 0, 0, 0, 0],
    [0, 0, 0, 1, 0, 0, 0, 0, 0, 0],
    [0, 0, 0, 0, 1, 0, 0, 0, 0, 0],
    [0, 0, 0, 0, 0, 1, 0, 0, 0, 0],
    [0, 0, 0, 0, 0, 0, 1, 0, 0, 0],
    [0, 0, 0, 0, 0, 0, 0, 1, 0, 0],
    [0, 0, 0, 0, 0, 0, 0, 0, 1, 0],
    [0, 0, 0, 0, 0, 0, 0, 0, 0, 1]
]

square = [
    [1, 1, 1, 1, 1, 1, 1, 1, 1, 1],
    [1, 0, 0, 0, 0, 0, 0, 0, 0, 1],
    [1, 0, 0, 0, 0, 0, 0, 0, 0, 1],
    [1, 0, 0, 0, 0, 0, 0, 0, 0, 1],
    [1, 0, 0, 0, 0, 0, 0, 0, 0, 1],
    [1, 0, 0, 0, 0, 0, 0, 0, 0, 1],
    [1, 0, 0, 0, 0, 0, 0, 0, 0, 1],
    [1, 0, 0, 0, 0, 0, 0, 0, 0, 1],
    [1, 0, 0, 0, 0, 0, 0, 0, 0, 1],
    [1, 1, 1, 1, 1, 1, 1, 1, 1, 1]
]

triangle = [
    [1, 0, 0, 0, 0, 0, 0, 0, 0, 0],
    [1, 1, 0, 0, 0, 0, 0, 0, 0, 0],
    [1, 1, 1, 0, 0, 0, 0, 0, 0, 0],
    [1, 1, 1, 1, 0, 0, 0, 0, 0, 0],
    [1, 1, 1, 1, 1, 0, 0, 0, 0, 0],
    [1, 1, 1, 1, 1, 1, 0, 0, 0, 0],
    [1, 1, 1, 1, 1, 1, 1, 0, 0, 0],
    [1, 1, 1, 1, 1, 1, 1, 1, 0, 0],
    [1, 1, 1, 1, 1, 1, 1, 1, 1, 0],
    [1, 1, 1, 1, 1, 1, 1, 1, 1, 1]
]

TRANSFORMATION FUNCTIONS:

def add(a, b):
    return np.logical_or(a, b).astype(int)

def subtract(a, b):
    return np.logical_and(a, np.logical_not(b)).astype(int)

def intersect(a, b):
    return np.logical_and(a, b).astype(int)

def invert(a):
    return np.logical_not(a).astype(int)

def reflect_horizontal(a):
    return np.flipud(a)

def reflect_vertical(a):
    return np.fliplr(a)

def reflect_diag(a):
    return a.T

EXAMPLE:

Using add(a, b) transformation

Target:
[[0,0,0,0,0,1,0,0,0,0], 
[0,0,0,0,0,1,0,0,0,0], 
[0,0,0,0,0,1,0,0,0,0],
[0,0,0,0,0,1,0,0,0,0], 
[0,0,0,0,0,1,0,0,0,0], 
[1,1,1,1,1,1,1,1,1,1],
[0,0,0,0,0,1,0,0,0,0], 
[0,0,0,0,0,1,0,0,0,0], 
[0,0,0,0,0,1,0,0,0,0],
[0,0,0,0,0,1,0,0,0,0]]

Solution:
def reconstructed():
    return add(line_horizontal, line_vertical)
\end{verbatim}

\subsection{Current trial prompt}

The 10×10 target array that the LLM is tasked with reconstructing for the current trial is appended to the end of the prompt. In this example, the current trial corresponds to Trial 3 of Experiment 1:

\begin{verbatim}
This is trial 3 of 14.

Target:
[
[0, 0, 0, 0, 0, 0, 0, 0, 0, 0],
[0, 1, 1, 1, 0, 0, 1, 1, 1, 0],
[0, 1, 1, 1, 0, 0, 1, 1, 1, 0],
[0, 1, 1, 1, 0, 0, 1, 1, 1, 0],
[0, 0, 0, 0, 0, 0, 0, 0, 0, 0],
[0, 0, 0, 0, 0, 0, 0, 0, 0, 0],
[0, 1, 1, 1, 0, 0, 1, 1, 1, 0],
[0, 1, 1, 1, 0, 0, 1, 1, 1, 0],
[0, 1, 1, 1, 0, 0, 1, 1, 1, 0],
[0, 0, 0, 0, 0, 0, 0, 0, 0, 0]
]
\end{verbatim} 

\subsection{Starter code prompt}

The starter code, comprising the empty \verb|reconstructed| function along with any previously defined helpers, is then appended. This portion of the prompt informs the LLM which helper functions are available, specifically those that were generated and successfully used in prior trials. In this example, the helper function \verb|make_thick_plus| was generated in an earlier trial and carried forward:

\begin{verbatim}
Do not include any comments or imports in your response. 
Respond by completing the following code:

--- Start ---

# You previously found these helpers useful (remove comment)

def make_thick_plus():
    h = line_horizontal
    v = line_vertical
    return add(add(h, reflect_horizontal(h)), add(v, reflect_vertical(v)))

# Define any new helpers here (remove comment)

def reconstructed():
    # Your code here (remove comment)

--- End ---
\end{verbatim}

\subsection{Trial history prompt (LLM-PS-H only)}

The history of target composites from all prior trials, along with an indication of whether the model successfully reconstructed each corresponding composite, is then appended. In this example, the model is on Trial 3 of Experiment 1, so the target composites from Trials 1 and 2 are included:

\begin{verbatim}
Below is the history of figures you've built on previous trials.

Target 1:

[
[0, 0, 0, 0, 1, 1, 0, 0, 0, 0],
[0, 0, 0, 0, 1, 1, 0, 0, 0, 0],
[0, 0, 0, 0, 1, 1, 0, 0, 0, 0],
[0, 0, 0, 0, 1, 1, 0, 0, 0, 0],
[1, 1, 1, 1, 1, 1, 1, 1, 1, 1],
[1, 1, 1, 1, 1, 1, 1, 1, 1, 1],
[0, 0, 0, 0, 1, 1, 0, 0, 0, 0],
[0, 0, 0, 0, 1, 1, 0, 0, 0, 0],
[0, 0, 0, 0, 1, 1, 0, 0, 0, 0],
[0, 0, 0, 0, 1, 1, 0, 0, 0, 0]
]

Built Correctly.

Target 2:

[
[1, 1, 1, 1, 1, 1, 1, 1, 1, 1],
[1, 0, 0, 0, 1, 1, 0, 0, 0, 1],
[1, 0, 0, 0, 1, 1, 0, 0, 0, 1],
[1, 0, 0, 0, 1, 1, 0, 0, 0, 1],
[1, 1, 1, 1, 1, 1, 1, 1, 1, 1],
[1, 1, 1, 1, 1, 1, 1, 1, 1, 1],
[1, 0, 0, 0, 1, 1, 0, 0, 0, 1],
[1, 0, 0, 0, 1, 1, 0, 0, 0, 1],
[1, 0, 0, 0, 1, 1, 0, 0, 0, 1],
[1, 1, 1, 1, 1, 1, 1, 1, 1, 1]
]

Built Correctly.
\end{verbatim}

\subsection{Refinement feedback prompt}

If the model generates code that does not exactly match the current target composite, an iterative refinement feedback prompt is appended in the subsequent iteration. This prompt includes the number of remaining attempts (5 attempts total), the incorrect program produced by the model, the resulting array it generated, and the target array for comparison:

\begin{verbatim}
Note: Your previous response was incorrect. 
This was attempt 1 of 5. You have 4 attempt(s) remaining.

Your code:
<incorrect LLM program here>

This is what your code produced:
<10x10 array produced by the incorrect program above>

This is the target you need to produce:
<current 10x10 target array>

Please try again.
\end{verbatim}

\subsection{Model 3a: Memoryless (LLM-PS)}

The memoryless LLM-PS model (LLM-PS) uses the same base prompts as the history-enabled LLM-PS-H model but omits any history-specific information, such as statements indicating that the task is part of a sequence. Additionally, the trial history prompt is excluded, meaning the model has no awareness of prior trials and operates solely on the current task.

\section{Examples of LLM Produced Reconstructions}

\subsection{Simple reconstruction - Experiment 1}

Composite 1 sample reconstruction:

\begin{verbatim}
	def make_double_vertical():
	    v1 = line_vertical
	    v2 = reflect_vertical(v1)
	    return add(v1, v2)

	def make_double_horizontal():
	    h1 = line_horizontal
	    h2 = reflect_horizontal(h1)
	    return add(h1, h2)

	def reconstructed():
	    return add(make_double_vertical(), make_double_horizontal())
\end{verbatim}

\subsection{Complex reconstruction - Experiment 1}

Composite 14 sample reconstruction:

\begin{verbatim}
	def make_double_horizontal():
	    return add(line_horizontal, reflect_horizontal(line_horizontal))

	def make_double_vertical():
	    v1 = line_vertical
	    v2 = reflect_vertical(v1)
	    return add(v1, v2)

	def make_x():
	    d1 = diagonal
	    d2 = reflect_vertical(d1)
	    return add(d1, d2)

	def make_corners():
	    return intersect(make_x(), square)

	def make_center_square():
	    return intersect(make_double_horizontal(), make_double_vertical())

	def make_full():
	    return invert(blank)

	def make_no_middle_rows():
	    return subtract(make_full(), make_double_horizontal())

	def make_vertical_no_middle():
	    return intersect(make_double_vertical(), make_no_middle_rows())

	def make_middle_rows():
	    return invert(make_no_middle_rows())

	def make_middle_rows_no_center():
	    return intersect(make_middle_rows(), invert(make_double_vertical()))

	def make_center_block():
	    return intersect(make_double_vertical(), make_middle_rows())

	def make_x_hole_center():
	    return subtract(make_x(), make_center_block())

	def reconstructed():
	    return add(make_x_hole_center(), add(make_vertical_no_middle(), make_middle_rows_no_center()))
\end{verbatim}

\subsection{Simple reconstruction - Experiment 2}

Composite 2 sample reconstruction:

\begin{verbatim}
	def make_x():
	    d1 = diagonal
	    d2 = reflect_vertical(d1)
	    return add(d1, d2)

	def inner_square():
	    return invert(square)

	def reconstructed():
	    return intersect(make_x(), inner_square())
\end{verbatim}

\subsection{Complex reconstruction - Experiment 2}

Composite 4 sample reconstruction:

\begin{verbatim}
	def make_x():
	    d1 = diagonal
	    d2 = reflect_vertical(d1)
	    return add(d1, d2)

	def make_double_vertical():
	    return add(line_vertical, reflect_vertical(line_vertical))

	def make_all_ones():
	    return invert(blank)

	def make_center_pixel():
	    return intersect(diagonal, line_horizontal)

	def make_extra_pixel():
	    return intersect(line_vertical, reflect_horizontal(line_horizontal))

	def reconstructed():
	    full = make_all_ones()
	    no_x = subtract(full, make_x())
	    return add(add(no_x, make_center_pixel()), make_extra_pixel())
\end{verbatim}

\section{Search Space Complexity}\label{si:search_size}

We assume a DSL over primitives $\mathcal{X}$ and transformations $\mathcal{T}$. Common to the experiments, the specific DSL consists of 6 geometric primitives ($|\mathcal{X}| = 6$), 3 binary transformation operators, and 4 unary transformation operators ($|\mathcal{T}| = 7$ total) (see Figure \ref{si:dsl}). 

\begin{figure}[ht]
  \begin{center}
    \includegraphics[width=0.9\columnwidth, height=0.15\textheight,keepaspectratio]{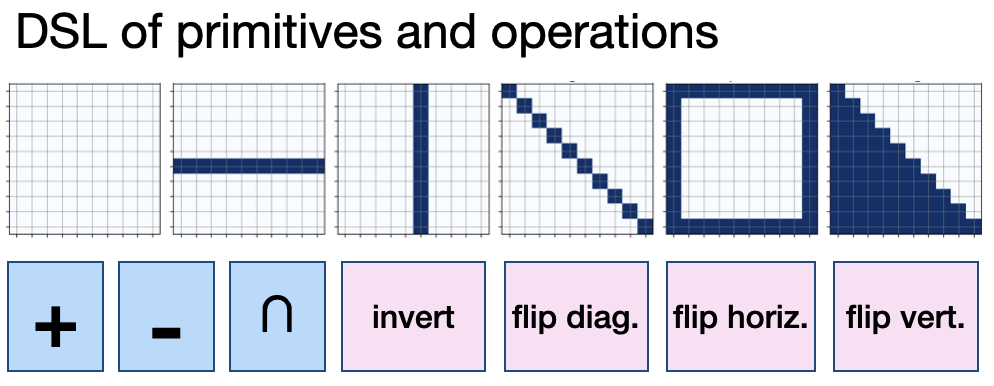}
  \end{center}
  \caption{The DSL of transformation and geometric primitives used in the two experiments consists of 6 geometric primitives, 3 binary operations, and 4 unary operations.}
  \label{si:dsl}
\end{figure}

Programs are represented as binary trees, where internal nodes correspond to binary operators, unary operators are applied to a single child, and leaves correspond to primitives or library helpers. 

We show that the number of syntactically distinct programs grows rapidly with tree depth, necessitating library learning.

\subsection{Search Space without library search}

At
depth $d$, the number of candidate programs $|\mathcal{H}_d|$ is bounded below by the number of full binary trees of depth $d$ with leaves drawn from $\mathcal{X}$:

\begin{equation}
    |\mathcal{H}_d| \;\geq\; |\mathcal{X}|^{2^d} \cdot
    |\mathcal{T}_{\mathrm{bin}}|^{2^d - 1}
\end{equation}

\noindent where $|\mathcal{T}_{\mathrm{bin}}| = 3$ is the number of binary operators. Even restricting to programs of depth $d \leq 4$, this lower bound exceeds $10^6$ candidate programs. In practice the space is larger still, as unary operators can be applied at any internal node, further multiplying the branching factor. Bottom-up
enumeration with observational equivalence pruning
\citep{albarghouthi2013recursive, udupa2013transit} reduces this space by collapsing programs that produce identical outputs on the empty canvas, but the effective search space remains super-exponential in program depth. This motivates the failure of the no-library baseline on later trials where target programs require depth exceeding 3. Without library learning to compress the search space the synthesizer exhausts its computational budget before reaching programs of sufficient depth to reconstruct complex targets.

\subsection{Characterizing the size of reachable space in PBT using stochastic exploration}\label{si:randomwalk}

To establish that the PBT domain is non-trivially large, we characterized the space reachable by unguided search. If a stochastic exploration over the program space finds thousands of unique patterns, this would show that the reachable space is sufficiently large to make the search hard, and motivate curriculum-based structure in our experimental procedure. 

We ran a symmetry-biased random walk over the DSL for 5 million steps. The walk maintains a bounded pool of programs; at each step, a program is sampled from the pool and a randomly selected operator is applied to produce a candidate. Candidates are accepted only if their rendered output has not been seen before and the output is symmetric across at least one axis, otherwise the candidate is discarded. The symmetry constraint was imposed to produce a subset of the reachable space that is more likely to be visually interesting by expressing a simple gestalt prior.

When the pool is full and a new program must be admitted, the longest program in the pool is evicted, keeping the pool biased toward shorter programs. The reported count is therefore a lower bound on true program diversity: the walk tallies only visually distinct, symmetric outputs, and the eviction policy further restricts exploration to the moderate-length region of the space.

\textbf{Discovery curve.} The walk discovered 8791 unique programs within $5 \times 10^6$ steps (Figure \ref{si:discovery}) with no indication of saturation. The distribution of discovered programs by program length (Figure \ref{si:prog_dist}) clarifies why: the eviction policy keeps shorter programs in the pool, compositions of those tend to produce symmetric outputs concentrated around 65–75 nodes, and the space at that length scale is large enough that the walk continues finding novel outputs at an accelerating rate throughout. Programs longer than roughly 150 nodes are discovered rarely, because they are evicted from the pool before they can be extensively explored. This further establishes that the reported count is a lower bound on the theoretical space size as the walk has not come close to exhausting even the moderate-length region of the space. The longer-program tail remains almost entirely unsampled.

\begin{figure}[ht]
  \begin{center}
    \includegraphics[width=\textwidth, height=0.7\textheight,keepaspectratio]{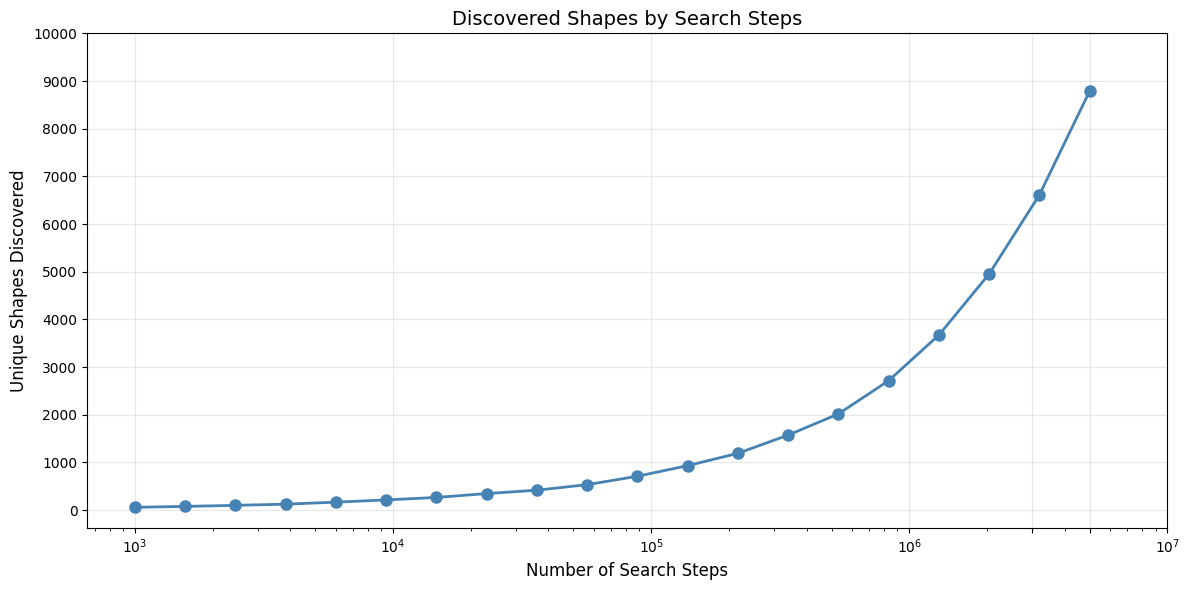}
  \end{center}
  \caption{Number of unique patterns discovered using stochastic explorations of the PBT domain using the DSL in Figure \ref{si:dsl}.}
  \label{si:discovery}
\end{figure}

\begin{figure}[ht]
  \begin{center}
    \includegraphics[width=\textwidth, height=0.7\textheight,keepaspectratio]{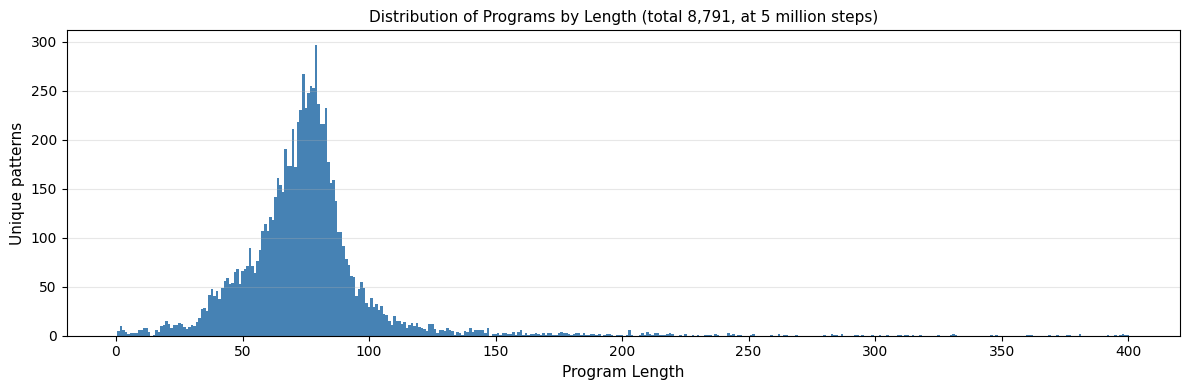}
  \end{center}
  \caption{Distribution of program lengths generated via random search, measured in number of raw DSL primitives.}
  \label{si:prog_dist}
\end{figure}

\clearpage

\textbf{Program diversity.} Figure \ref{si:shapemosaic} shows a random mosaic of 225 sampled outputs. Without any curriculum structure to guide the search, a symmetry-biased random walk over the DSL discovers 8791 distinct visual programs of varying program lengths. This makes two points concrete. First, library learning algorithms operating in this domain face a candidate space far larger than either experiment alone would suggest, even under the conservative lower-bound conditions of this random walk. Second, the structured curricula in Experiments 1 and 2 are not arbitrary samples from the reachable space: they are deliberate selections designed to expose latent generative structure that unguided search does not reliably surface. The visual heterogeneity of the mosaic confirms that this DSL underlies a large and varied space, where curriculum design is doing principled work.

\begin{figure}[ht]
  \begin{center}
    \includegraphics[width=\textwidth, height=0.7\textheight,keepaspectratio]{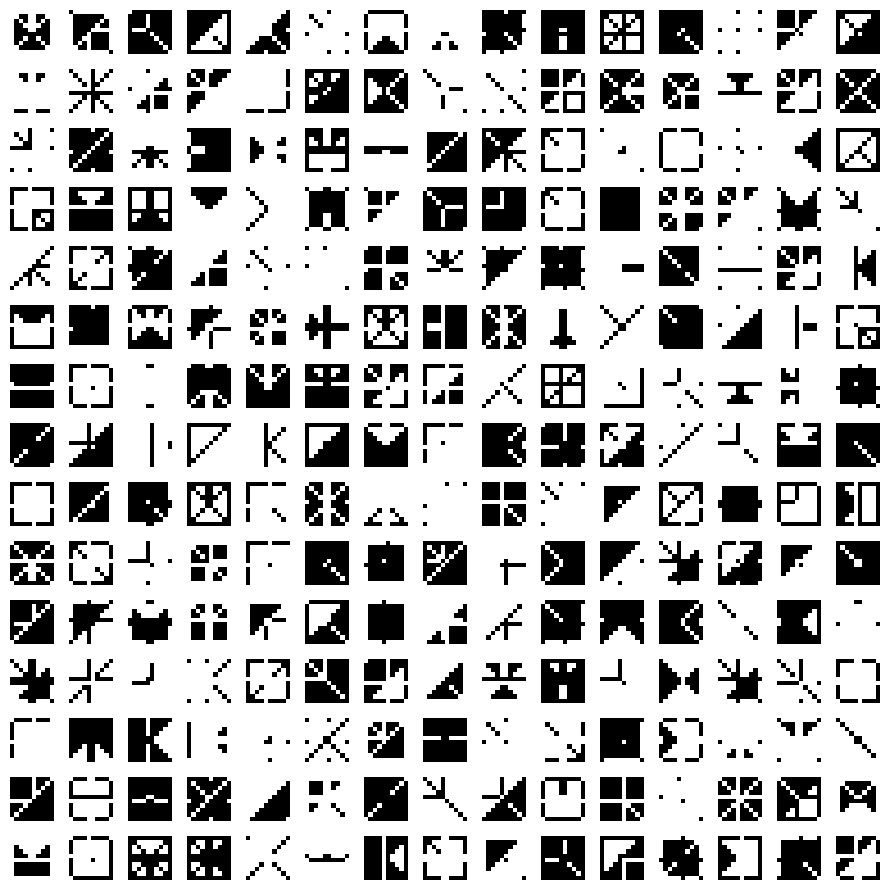}
  \end{center}
  \caption{225 shapes (2.6\%) sampled from the 8,791 shapes generated by random walk with $5 \times 10^6$ steps.}
  \label{si:shapemosaic}
\end{figure}

\subsection{Program length distributions across curricula} 


\section{Sequence Meta-structure in Experiment Design}\label{si:curriculum}

\textbf{Definition:} A \textbf{curriculum} is defined as an ordered sequence of targets $\mathbf{P} = (p_1, p_2, \ldots, p_T)$ presented on consecutive trials.

We next consider different ways in which a curriculum can be structured.

\textbf{Definition:}
We define \textbf{Meta-Structure} $\mathcal{G}$ of a curriculum $\mathbf{P}$ as
\begin{quote}
a generative process that imposes structured dependencies on solutions, either through preceding solutions or shared latent sub-programs,  specifying how each $p_t$ can be generated given a subset of the sequence $(p_{t-n}, \ldots, p_{t-1})$, where $n$ specifies how many steps back the dependencies extend.
\end{quote}

\textbf{Definition:}
We define a \textbf{sequential} curriculum as
\begin{quote}
a process by which each solution is obtained by applying a transformation $\tau \in \mathcal{T}$ to a preceding solution $p_{t-n}$ and primitive $x \in \mathcal{X}$:
\begin{equation}
   \mathcal{G}_{\mathrm{seq}}: \quad p_{t} = \tau(p_{t-n},\, x),
    \qquad \tau \in \mathcal{T},\ x \in \mathcal{X},\ 0 < n < t.
\end{equation}
\end{quote}

\textbf{Definition:}
We define a \textbf{long-range} curriculum as
\begin{quote}
a sequential curriculum in which dependencies extend beyond the immediately preceding trial, i.e.\ $n > 1$ in $\mathcal{G}_{\mathrm{seq}}$.
\end{quote}

The distinction between sequential and long-range curricula affects how readily the latent meta-structure $\mathcal{G}$ can be inferred: sequential dependencies ($n = 1$) are local and discoverable from adjacent trials, whereas long-range dependencies ($n > 1$) require 
the learner to integrate information across a broader temporal window. Both forms of meta-structure are expected to elicit the addition of final targets as helpers.

\textbf{Definition:}
We define a \textbf{shared helper} curriculum as a generative procedure which derives programs from a common reusable sub-program $h \in \mathcal{P}$. Formally,
\begin{quote}
a group of $K$ consecutive programs $\{p_{t}, \ldots, p_{t+K-1}\}$ follows a shared helper curriculum with respect to $h$ if $h$ appears as a sub-program of every solution in the group:
\begin{equation}
    \mathcal{G}_{\mathrm{hlp}}: \quad h \sqsubseteq p_{t+k},
    \qquad k = 0, \ldots, K-1,
\end{equation}
where $\sqsubseteq$ denotes the sub-program relation.
\end{quote}

Next, we generalize this structure.

\textbf{Definition:}
We define an \textbf{operator-group} curriculum as
\begin{quote}
a process by which a group of $K$ consecutive programs $\{p_{t}, \ldots, p_{t+K-1}\}$ are all generated by applying distinct transformations from a fixed operator set, in a given order and exactly once, to a shared sub-program $h \in \mathcal{P}$ and a primitive $x \in \mathcal{X}$.
\end{quote}

\textbf{Definition:}
We define a specific \textbf{4-operator-group} curriculum where
\begin{quote}
given a shared sub-program $h \in \mathcal{P}$ and a primitive $x \in \mathcal{X}$, the group consists of four programs generated by a fixed set of operators:
\begin{equation}
\begin{split}
    \mathcal{G}_{\mathrm{grp}}: \quad \{p_{t+k}\}_{k=0}^{3} = \bigl\{\,
        &\textsc{Add}(h,\, x),\quad
        \textsc{Subtract}(h,\, x),\\
        &\textsc{Overlap}(h,\, x),\quad
        \textsc{Add}\!\bigl(\textsc{Invert}(h),\, x\bigr)
    \,\bigr\},
\end{split}
\end{equation}
subject to the constraint that $h$ is a parsimonious shared structure: no target in the group can be derived from any other target in the group by a program shorter than that prescribed by $\mathcal{G}_{\mathrm{grp}}$.
\end{quote}

The operator-group meta-structure generalises straightforwardly to different operator sets. 

\subsection{Experiment 1 Meta-Structure and trial sequence}

The goal of Experiment 1 is to establish that people rely on reusable helpers, as predicted by library learning models.

\begin{table}[ht]
\centering
\begin{tabular}{lll}
\hline
\textbf{Pattern} & \textbf{Program} & \textbf{Type} \\
\hline
$P_1$ & \texttt{fat\_cross()} & - \\
$P_2$ & \texttt{add($P_1$, square)} & Sequential \\
$P_3$ & \texttt{invert($P_2$)} & Sequential \\
\hline
$P_4$ & \texttt{subtract(square, Diagonal\_Cross)} & Helper \\
$P_5$ & \texttt{add(square, Diagonal\_Cross)} & Helper \\
$P_6$ & \texttt{invert($P_5$)} & Sequential \\
\hline
$P_7$ & \texttt{add(Diagonal\_Cross, $P_1$)} & Long range \\
$P_8$ & \texttt{intersect(Diagonal\_Cross, $P_1$)} & Long range \\
\hline
$P_9$  & \texttt{add($P_8$, square)} & Sequential \\
$P_{10}$ & \texttt{invert($P_9$)} & Sequential \\
$P_{11}$ & \texttt{add($P_8$, invert($P_1$))} & Long range \\
$P_{12}$ & \texttt{add(invert($P_{11}$), square)} & Sequential \\
\hline
$P_{13}$ & \texttt{subtract(Diagonal\_Cross, $P_8$)} & Long range \\
$P_{14}$ & \texttt{subtract($P_7$, $P_8$)} & Long range \\
\hline
\end{tabular}
\caption{Curriculum structure for Experiment~1. Each row specifies a target pattern ($P_t$), its program expression, and its structural type. 
Sequential trials depend on the immediately preceding target, and long range trials depend on earlier targets. Targets $P_4$ and $P_5$ derive from a common helper \texttt{Diagonal\_Cross}.}
\label{tab:curriculum_e1}
\end{table}

\begin{figure}[ht]
  \begin{center}
    \includegraphics[width=\textwidth, height=0.8\textheight,keepaspectratio]{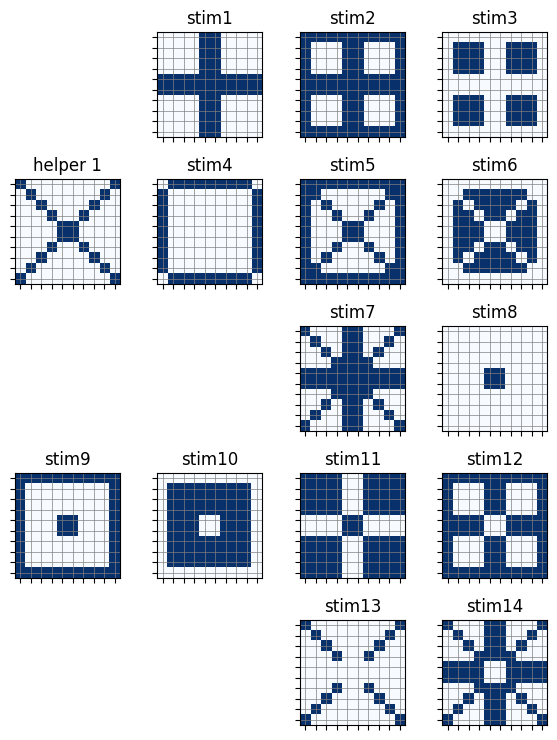}
  \end{center}
  \caption{Target sequence from Experiment 1.}
  \label{si:e1}
\end{figure}

\textbf{Prediction:}
In Experiment~1, the curriculum is dominated by sequential and long-range dependencies (Table~\ref{tab:curriculum_e1}), and therefore participants sensitive to this latent meta-structure are expected to increasingly retain target programs themselves as helpers.

\begin{figure}[ht]
  \begin{center}
    \includegraphics[width=\textwidth, height=0.7\textheight,keepaspectratio]{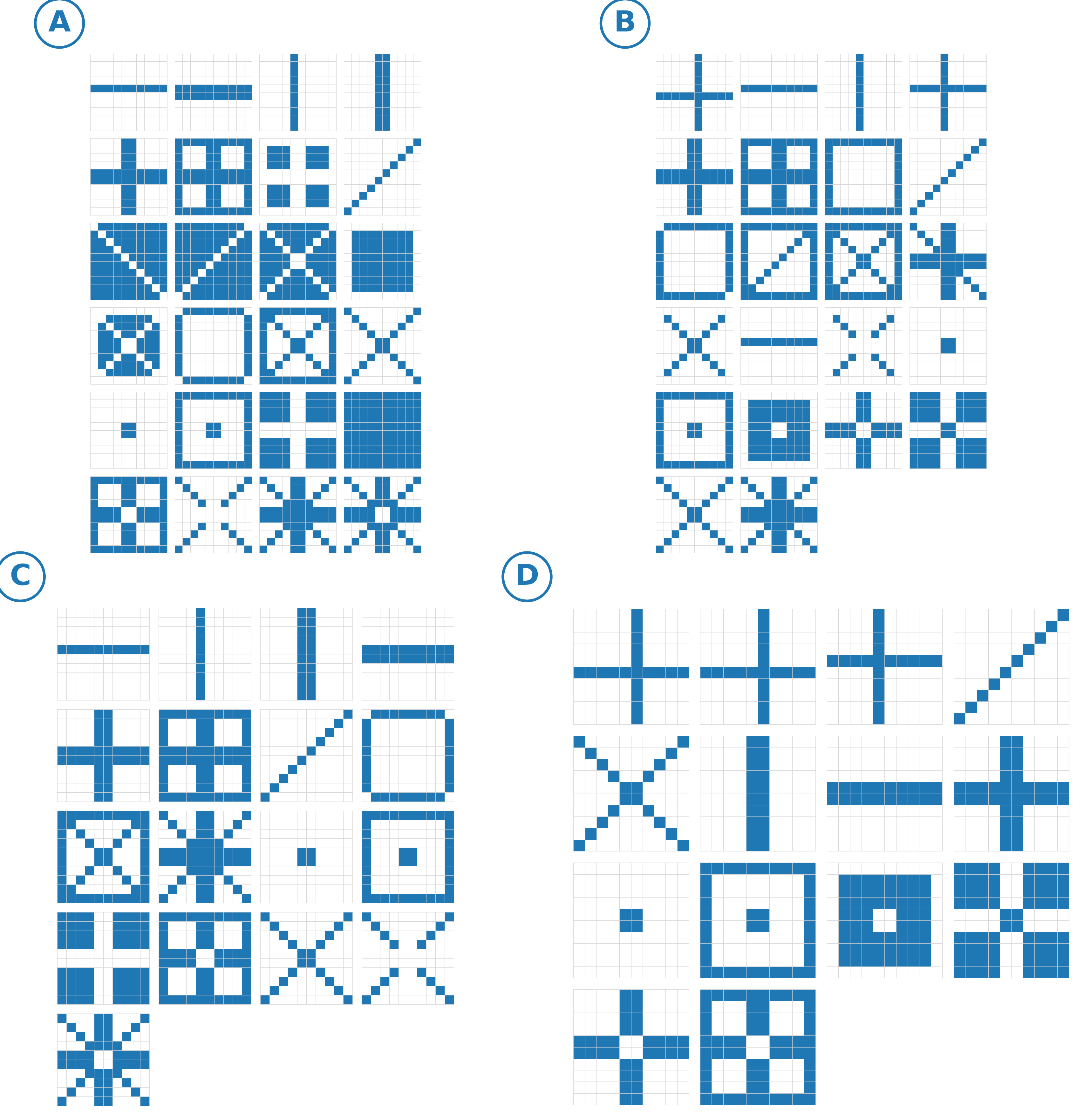}
  \end{center}
  \caption{Examples of helpers made by individual users in Experiment 1. Panels A, B, C, D show helpers from different users. Each user's library becomes biased toward including the final targets as helpers as the experiment progressed.}
  \label{si:e1_helpers}
\end{figure}

\clearpage

\subsection{Experiment 2 Meta-Structure and trial sequence}

Our goal is to show that people are sensitive to the latent generative process itself (the meta-structure). To do this, Experiment~2 uses a target sequence generated using \textbf{4-operator-group} curriculum. 
The same meta-structure is used for every target in Experiment 2.
Here each trial is generated using a simple transformation of one helper.

\begin{figure}[ht]
  \begin{center}
    \includegraphics[width=\textwidth, height=0.7\textheight,keepaspectratio]{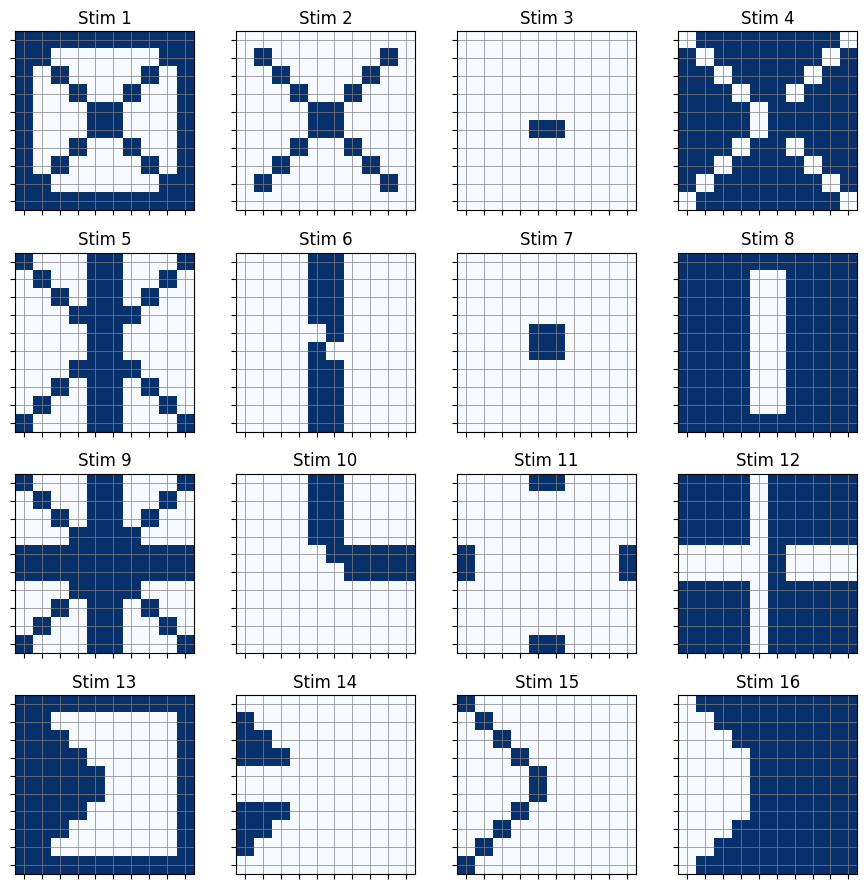}
  \end{center}
  \caption{Target sequence from Experiment 2.}
  \label{si:e2}
\end{figure}

\clearpage

\begin{figure}[ht]
  \begin{center}
    \includegraphics[width=\textwidth, height=0.7\textheight,keepaspectratio]{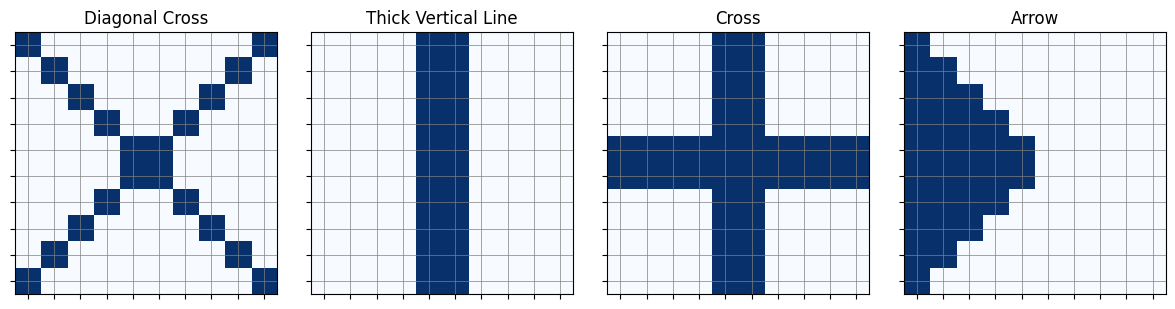}
  \end{center}
  \caption{Helpers used in 4-operator-group by design for Experiment 2.}
  \label{si:e2_helpers}
\end{figure}

\textbf{Prediction 1:} As it takes a few examples to infer the underlying meta-structure, helpers people are expected to make will diverge from by-design helpers, but should increasingly become biased toward the ground-truth set.

\textbf{Prediction 2:} Helpers people make will be biased toward representing intermediate structure rather than final targets, as in E1.




\begin{figure}[ht]
  \begin{center}
    \includegraphics[width=\textwidth, height=0.7\textheight,keepaspectratio]{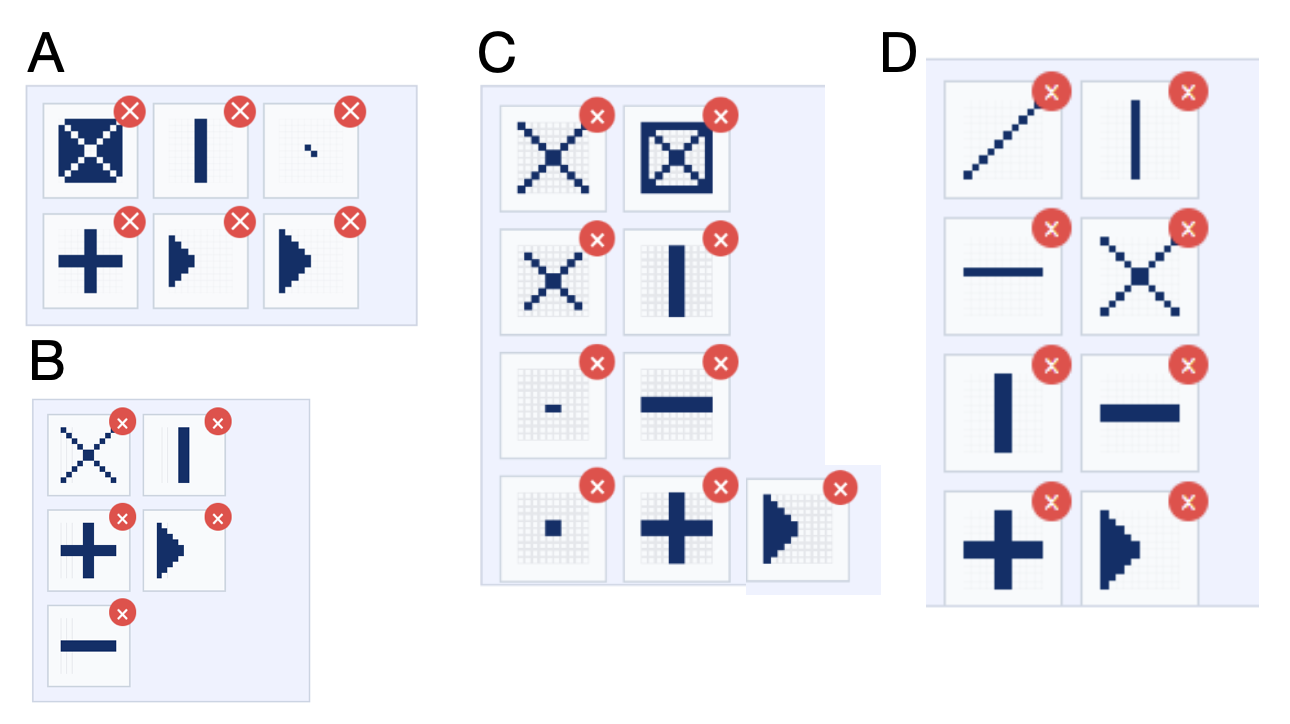}
  \end{center}
  \caption{Examples of helpers made by individual users in Experiment 2. Panels A, B, C, D show helpers from different users. Each user's library is biased toward including the by-design helpers (in contrast to final targets).}
  \label{sample-figure}
\end{figure}

\clearpage

\newpage

\end{document}

%% file: tmp.tex
We study the complexity of learning a single reusable abstraction (helper)
that maximizes compression utility over a corpus of programs. We show that
even a highly restricted version of this problem is NP-complete.

\par\noindent
\paragraph{High-level intuition.}
The restricted setting removes almost all sources of complexity usually
associated with library learning. Programs have no internal structure:
each program is just a flat tuple of symbols, and a helper can only specify
that certain coordinates must take fixed values. Thus, learning a helper
reduces to selecting a subset of positions that agree across many programs.

Despite this extreme simplification, the problem remains hard. The reason is
that the learner must jointly choose (i) which positions to constrain, and
(ii) which programs will satisfy those constraints. These two choices
interact combinatorially: constraining more positions makes matches rarer,
while relaxing constraints increases coverage. This trade-off mirrors dense
substructure problems such as biclique detection.

The general version of library learning adds further complexity. Programs
are no longer flat tuples but trees, so a helper must identify a reusable
subtree rather than fixed positions. Variables (or \emph{holes}) allow parts
of the helper to vary across occurrences, so the same pattern can match
many different concrete subprograms (e.g., the same structure with different
inputs)~\cite{dechter2013bootstrap,ellis2023dreamcoder,cao2023babble,bowers2023top}. This increases flexibility but also enlarges the search
space. Moreover, matching is no longer exact: a helper may apply at many
different locations inside a program, and one must search for all places
where the pattern fits~\cite{bowers2023top}. The result below shows that even before these
additional layers, the core combinatorial problem is already NP-hard.

\par\noindent
\paragraph{Syntax and programs.}
Let $\Sigma$ be a finite set of symbols, and let $f$ be a fixed $n$-ary
constructor. Every program is defined as
\[
p \equiv f(c_1, \dots, c_n),
\]
where each $c_i \in \Sigma$. The arity $n$ is fixed and identical for all
programs.

\par\noindent
\paragraph{Corpus.}
We are given a corpus
\[
\mathcal C \equiv \{p^{(1)}, \dots, p^{(N)}\},
\]
where each program is written as
\[
p^{(j)} \equiv f(c_1^{(j)}, \dots, c_n^{(j)}).
\]

\par\noindent
\paragraph{Helpers (abstractions).}
A reference tuple $(a_1, \dots, a_n) \in \Sigma^n$ is part of the input.
A helper is defined by a subset of positions
\[
S \subseteq \{1, \dots, n\}.
\]

A program $p^{(j)}$ matches $S$ if
\[
\forall i \in S,\quad c_i^{(j)} = a_i.
\]

Let
\[
\mathrm{occ}(S) = \#\{ j : p^{(j)} \text{ matches } S \}, 
\]  the number of corpus programs consistent with the helper.

Define the utility
\[
U(S) = |S| \cdot \mathrm{occ}(S).
\]

\par\noindent
\paragraph{Example.}
Let $\Sigma = \{A, B, X, Y\}$ and consider programs of arity $3$:
\[
\begin{array}{c|ccc}
 & 1 & 2 & 3 \\
\hline
p^{(1)} & A & A & X \\
p^{(2)} & A & A & Y \\
p^{(3)} & A & B & X \\
\end{array}
\]
Fix the reference tuple $(A, A, X)$.

If $S = \{1,2\}$, then $p^{(1)}$ and $p^{(2)}$ match, so
$\mathrm{occ}=2$ and $U=4$.
If $S = \{1\}$, all programs match, so $U=3$.

\par\noindent
\paragraph{Decision problem.}
\begin{quote}
\textsc{Best-Single-Helper} \\
\textbf{Input:} A corpus $\mathcal C$, a reference tuple $(a_1,\dots,a_n)$, and an integer $\omega$. \\
\textbf{Question:} Does there exist $S \subseteq \{1,\dots,n\}$ such that
\[
U(S) \ge \omega \, ?
\]
\end{quote}

\begin{proposition}
\textsc{Best-Single-Helper} is NP-complete.
\end{proposition}

\begin{proof}

\par\noindent
\paragraph{Membership in NP.}
Given a subset $S$, we can compute $\mathrm{occ}(S)$ by scanning all programs
and checking the constraint at each position, which takes time $O(Nn)$.
Thus $U(S)$ can be computed in polynomial time.

\par\noindent
\paragraph{High-level idea of the reduction.}
We reduce from \textsc{Maximum-Edge-Biclique}: given a bipartite graph,
find subsets $S \subseteq L$ and $T \subseteq R$ such that all edges between
them exist, maximizing $|S|\cdot |T|$. The decision version asks whether
there exists such a biclique with $|S|\cdot|T| \ge k$ , which is known to be NP-complete \citep{peeters2003maximum}.

Positions will correspond to vertices in $L$, and programs to vertices in $R$.
Selecting positions enforces that only programs adjacent to all those vertices
can match, yielding a biclique structure.

\par\noindent
\paragraph{Reduction.}
Let $G = (L \cup R, E)$ be a bipartite graph with
\[
L = \{u_1, \dots, u_n\}, \quad R = \{v_1, \dots, v_m\},
\]
and let $k$ be the target.

Set $M \equiv n + 1$.

\par\noindent
\paragraph{Construction.}
We construct a corpus $\mathcal C$ of size $mM$.
Let $\Sigma$ contain symbols $a_1, \dots, a_n$ and fresh symbols
$b_{ij\ell}$.

For each $v_j \in R$ and each $\ell \in \{1,\dots,M\}$, define
\[
p^{(j,\ell)} \equiv f(c_1^{(j,\ell)}, \dots, c_n^{(j,\ell)}),
\]
where
\[
c_i^{(j,\ell)} =
\begin{cases}
a_i & \text{if } (u_i, v_j) \in E, \\
b_{ij\ell} & \text{otherwise}.
\end{cases}
\]

Set $\omega \equiv Mk$.

This construction is polynomial in the size of $G$.

\par\noindent
\paragraph{Example (matrix view of the construction).} Consider a bipartite graph with 

\[ 
L = \{u_1, u_2\}, \quad R = \{v_1, v_2\}, 
\] 

and edges 

\[ 
(u_1,v_1), (u_2,v_1), (u_1,v_2). 
\]

Thus, $(u_2, v_2) \notin E$. Let $M = 3$. The construction builds a corpus of programs of arity $2$. We can represent it as: 

\[ 
\begin{array}{c|cc} & i=1 & i=2 \\ \hline p^{(1,1)} & a_1 & a_2 \\ p^{(1,2)} & a_1 & a_2 \\ p^{(1,3)} & a_1 & a_2 \\ \hline p^{(2,1)} & a_1 & b_{2,2,1} \\ p^{(2,2)} & a_1 & b_{2,2,2} \\ p^{(2,3)} & a_1 & b_{2,2,3} \\ \end{array} 
\]

\par\noindent
\paragraph{Forward direction (good biclique $\Rightarrow$good helper).}
Suppose there exists a biclique $S \subseteq L$, $T \subseteq R$
such that $|S|\cdot|T| \ge k$.
Let $S$ also denote the corresponding set of indices. For any $v_j \in T$ and any $i \in S$, we have $(u_i,v_j)\in E$,
so $c_i^{(j,\ell)} = a_i$ for all $\ell$.
Thus all $M$ programs associated with $v_j$ match $S$.

Therefore

\[
\mathrm{occ}(S) \ge M|T|, \quad
U(S) \ge M|S||T| \ge Mk = \omega.
\]

\par\noindent
\paragraph{Reverse direction (good helper$\Rightarrow$good biclique).}
Suppose there exists $S \subseteq \{1,\dots,n\}$ such that
\[
U(S) \ge Mk.
\]

Define
\[
S' = \{u_i : i \in S\}.
\]

A program $p^{(j,\ell)}$ matches $S$ if and only if for all $i \in S$,
we have $(u_i,v_j)\in E$. Therefore: If $v_j$ is adjacent to all vertices in $S'$, then all $M$ programs
  $p^{(j,\ell)}$ match $S$; Otherwise, if $v_j$ is missing an edge to some $u_i \in S'$, then for that $i$,
  $c_i^{(j,\ell)} \neq a_i$, so none of the programs $p^{(j,\ell)}$ match $S$.

Let
\[
T = \{v_j \in R : (u_i,v_j)\in E \ \forall u_i \in S'\}.
\]
Then
\[
\mathrm{occ}(S) = M|T|.
\]

Hence
\[
U(S) = |S|\cdot \mathrm{occ}(S) = M|S||T| \ge Mk,
\]
which implies $|S||T| \ge k$.
Thus $(S',T)$ forms a biclique of size at least $k$.

\par\noindent
\paragraph{Conclusion.}
We have shown a polynomial-time reduction from
\textsc{Maximum-Edge-Biclique}, and the problem is in NP,
so \textsc{Best-Single-Helper} is NP-complete.
\end{proof}